\documentclass[11pt]{article}
\usepackage[thmeqnumber]{matus}
\setlist[enumerate]{leftmargin=2em}
\setlist[itemize]{leftmargin=2em}

\usepackage{caption}
\def\tv{\tilde{v}}
\def\hcRi{\hcR^{(i)}}
\def\ffi{f^{(i)}}
\def\ggi{g^{(i)}}
\newcommand\ff[1]{f^{(#1)}}
\def\hy{\hat y}
\def\barU{\overline U}
\def\baru{\overline u}
\def\barUi{\overline U_{\!\infty}}
\def\barcR{\overline{\cR}}
\def\barcRz{\overline{\cR}_{\textup{z}}}
\def\bcRz{\overline{\cR}_{\textup{z}}}
\def\klb{\cK_{\textup{bin}}}
\def\epsopt{\eps_{\textup{gd}}}
\def\radopt{R_{\textup{gd}}}

\usepackage{marginnote,tikzsymbols,totcount}
\usepackage{bm}
\newtotcounter{jjj}

\title{\textbf{Early-stopped neural networks are consistent}}

\author{Ziwei Ji\qquad{}\quad{}Justin D. Li\qquad{}\quad{}Matus Telgarsky\\
\texttt{<\{\href{mailto:ziweiji2@illinois.edu}{ziweiji2},\href{mailto:jdli3@illinois.edu}{jdli3},\href{mailto:mjt@illinois.edu}{mjt}\}@illinois.edu>}\\
University of Illinois, Urbana-Champaign}

\date{}

\begin{document}

\maketitle

\begin{abstract}
  This work studies the behavior of shallow ReLU networks trained with the logistic
  loss via gradient descent
  on binary classification data where the underlying data distribution is general,
  and the (optimal) Bayes risk is not necessarily zero.
  In this setting, it is shown that gradient descent with early stopping
  achieves population risk arbitrarily close to optimal in terms of not just
  logistic and misclassification losses, but also in terms of calibration,
  meaning the sigmoid mapping of its outputs approximates the true underlying
  conditional distribution arbitrarily finely.
  Moreover, the necessary iteration, sample, and architectural complexities of this
  analysis all scale
  naturally with a certain complexity measure of the true conditional model.
  Lastly, while it is not shown that early stopping is necessary, it is shown that
  any univariate classifier satisfying a \emph{local interpolation property} is
inconsistent.
\end{abstract}

\section{Overview and main result}

Deep networks trained with gradient descent seem to have no trouble adapting
to arbitrary prediction problems, and are steadily displacing stalwart methods across
many domains.
In this work,
we provide a mathematical basis for this good performance
on arbitrary binary classification problems,
considering the simplest possible networks:
shallow ReLU networks where only the inner (input-facing) weights are trained
via vanilla gradient descent with a constant step size.
The central contributions are as follows.

\begin{enumerate}
  \item
    \textbf{Fully general classification tasks.}
    The joint distribution generating the $(x,y)$ pairs only requires $x$ to be bounded,
    and is otherwise arbitrary.  In particular, the underlying distribution may be noisy,
    meaning the true conditional model of the labels, $\Pr[Y=1 | X=x]$, is arbitrary.
    
    In this setting, we show that as data, width, and training time increase,
    the logistic loss \emph{measured over the population} converges
    to optimality over all measurable functions, which moreover implies that
    the induced conditional model (defined by a sigmoid mapping) converges to the true model,
    and the population misclassification rate also converges to optimality.
This is in contrast with prior analyses of gradient descent, which either only
    consider the training risk
    \citep{allen_deep_opt,du_iclr,zou_deep_opt,oymak_moderate,song_quadratic},
    or can only handle restricted conditional models
    \citep{allen_3_gen,arora_2_gen,cao_deep_gen,nitanda_refined,ziwei_ntk,gu_polylog}.

  \item
    \textbf{Adaptivity to data simplicity.}
    The required number of data samples, network nodes,
    and gradient descent iterations all shrink
    if the \emph{distribution}
    satisfies a natural notion of simplicity: the true conditional model $\Pr[Y=1|X=x]$ is 
    approximated well by a low-complexity infinite-width random feature model.

\end{enumerate}

Rounding out the story and contributions,
firstly we present a brief toy univariate model hinting towards
the necessity of early stopping: concretely,
any univariate predictor satisfying a \emph{local interpolation property} can not achieve
optimal test error for noisy distributions.
Secondly, our analysis is backed by a number of lemmas that could
be useful elsewhere; amongst these are a
\emph{multiplicative error} property of the logistic loss,
and separately a technique to control the effects of large network
width over not just a finite sample, but over the entire sphere.

\subsection{Main result: optimal test error via gradient descent}

The goal in this work is to minimize the logistic risk over the population:
letting $\mu$ denote an arbitrary Borel measure over $(x,y)$ pairs with compactly-supported 
marginal $\mu_x$ and
conditional $p_y$, with a data sample $((x_i,y_i))_{i=1}^n$, and a function $f$, define the
logistic loss,
empirical logistic risk, 
and logistic risk respectively as
\[
  \ell(r) := \ln(1+e^{-r}),
  \qquad
  \hcR(f) := \frac 1 n \sum_{k=1}^n \ell(y_kf(x_k)),
  \qquad
  \cR(f) := \bbE_{x,y} \ell(yf(x)).
\]
We use the logistic loss not only due to its practical prevalence,
but also due to an interesting multiplicative error property which strengthens our main results
(cf. \Cref{fact:logistic:error} and \Cref{fact:main}),
all while being Lipschitz.

We seek to make the risk $\cR(f)$ as small as possible: formally, we compare against
the Bayes risk
\[
  \barcR := \inf\cbr{ \cR(f)\ {}:\ {}\textup{measurable } f\!:\!\R^d \to \R}.
\]
While competing with $\barcR$ may seem a strenuous goal, in fact it simplifies many 
aspects of the learning task.
Firstly, due to the universal approximation properties of neural networks
\citep{funahashi_apx,nn_stone_weierstrass,cybenko,barron_nn},
we are effectively working over the space of all measurable functions already.
Secondly, as will be highlighted in the main result below,
via the theory of classification calibration
\citep{zhang_convex_consistency,bartlett_jordan_mcauliffe},
competing with the Bayes (convex) risk also recovers the true conditional model,
and minimizes the misclassification loss;
this stands in contrast with the ostensibly more modest goal of minimizing misclassification
over a restricted class of predictors, namely the \emph{agnostic learning} setting,
which suffers a variety of computational and statistical obstructions
\citep{pmlr-v119-goel20a,surbhi_hardness,ohad_relu_gd,frei_agnostic}.

Our predictors are shallow ReLU networks, trained via gradient descent --- the simplest
architecture which is not convex in its parameters, but satisfies universal approximation.
In detail,
letting $(a_j)_{j=1}^m$
be uniformly random $\pm 1$ signs, $(w_j)_{j=1}^m$ with $w_j\in\R^d$ be standard Gaussians,
and $\rho>0$ be a \emph{temperature},
we predict on an input $x\in\R^d$ with
\[
  f(x;\rho, a,W) := f(x;W) := \frac {\rho}{\sqrt m}\sum_{j=1}^m a_j \srelu(w_j^\T x),
\]
where $\srelu(z):=\max\{0,z\}$ is the ReLU; since only $W$ is trained, both $\rho$ and $a$ are
often dropped.  To train, we perform gradient descent with a constant step size
on the empirical risk:
\[
  W_{i+1} := W_i - \eta \nhcR(W_i),
  \qquad\textup{where }
  \hcR(W) := \hcR\del{ x\mapsto f(x; W) }.
\]
Our guarantees are for an iterate with small empirical risk and small norm:
$W_{\leq t} := \argmin\{ \hcR(W_i) : i\leq t, \|W_i-W_0\|\leq \radopt\}$,
where $\radopt$ is our \emph{early stopping radius}:
if $\radopt$ is guessed correctly, our rates improve, but our analysis also
handles the case $\radopt=\infty$ where no guess is made, and indeed this is used
in our final consistency analysis
(a pessimistic, fully general setting).

Our goal is to show that this iterate $W_{\leq t}$ has approximately optimal
\emph{population} risk: $\cR(W_{\leq t}) \approx \barcR$.
Certain prediction problems may seem simpler than others, and we want our analysis 
to reflect this while abstracting away as many coincidences of the training process as
possible.  Concretely, we measure simplicity via the performance and complexity
of an infinite-width random feature model over the true distribution, primarily based on the
following considerations.
\begin{itemize}
  \item
    By measuring performance over the population,
    random effects of the training sample are removed, and it is impossible
    for the random feature model to simply revert to memorizing data,
    as it never sees that training data.
  \item
    The random feature model has infinite width, and via sampling can be used as a benchmark
    for all possible widths simultaneously, but is itself freed from coincidences
    of random weights.
\end{itemize}

In detail, our infinite-width random feature model is as follows.
Let $\barUi :\R^d\to\R^d$ be an (uncountable) collection of weights
(indexed by $\R^d$), and define a prediction mapping via
\[
  f(x;\barUi) := \int \ip{\barUi(v)}{x\1[v^\T x\geq 0]}\dif\cN(v),
  \qquad
  \text{whereby }
  \cR(\barUi) := \cR(x\mapsto f(x;\barUi)).
\]
Note that for each Gaussian random vector $v\sim \cN$, we construct a \emph{random feature}
$x\mapsto x\1[v^\T x\geq 0]$.  This particular choice is simply the gradient of a corresponding
ReLU $\nabla_v \srelu(v^\T x)$, and is motivated by the NTK literature
\citep{jacot_ntk,li_liang_nips,du_iclr}.
A similar object has appeared before in NTK convergence analyses
\citep{nitanda_refined,ziwei_ntk},
but the conditions on $\barUi$ were always strong (e.g., data separation with a margin).

What, then, does it mean for the data to be simple?  In this work, it is when there exists
a $\barUi$ with $\cR(\barUi)\approx \barcR$, and moreover $\barUi$ has low norm; for technical
convenience, we measure the norm as the maximum over individual weight norms,
meaning $\sup_v \|\barUi(v)\|$.  To measure approximability, for sake of interpretation,
we use the \emph{binary Kullback-Leibler divergence (KL)}:
defining a conditional probability model $\phi_\infty$ corresponding to $\barUi$ via
\[
  \phi_\infty(x) := \phi(f(x;\barUi)), \qquad\textup{where } \phi(r) := \frac {1}{1+\exp(-r)},
\]
then the binary KL can be written as
\[
  \klb(p_y, \phi_\infty) :=
  \int\del{p_y \ln\frac {p_y}{\phi_\infty} + (1-p_y)\ln\frac {1-p_y}{1-\phi_\infty}}\dif\mu_x
  = \cR(\barUi) - \barcR.
\]
This relationship between binary KL and the excess risk is a convenient property
of the logistic loss, which immediately implies calibration as a consequence of achieving
the optimal risk. 

The pieces are all in place to state our main result.

\begin{theorem}\label{fact:main}
  Let width $m\geq \ln(emd)$,
  temperature $\rho > 0$,
  and reference model
  $\barUi$ be given with $R := \max\{4, \rho, \sup_v\|\barUi(v)\|\}<\infty$,
  and define a corresponding conditional model $\phi_\infty(x) := \phi(f(x;\barUi))$.
  Let optimization accuracy $\epsopt$ and radius $\radopt\geq R/\rho$
  be given,
  define effective radius
  $B := \min\cbr[1]{\radopt,\ {}\frac{3R}{\rho} + \frac{4e}{\rho} \sqrt{t}\sqrt{e^{\tau_0}\cR(\barUi) + R\tau_n} }$,
  and generalization, linearization, and sampling errors
  $(\tau_n,\tau_1,\tau_0)$ as
  \begin{align*}
    \tau_n :=
\tcO\del{\frac{(d\ln(1/\delta))^{3/2}}{\sqrt n}}\!\!,
\ \tau_1 :=
\tcO\del{\frac{\rho B^{4/3}\sqrt{d\ln(1/\delta)}}{m^{1/6}}}\!\!,
\ \tau_0
    :=
\tcO\del{\rho \ln(1/\delta) + \frac {\sqrt{d\ln(1/\delta)}}{m^{1/4}}}\!\!,
  \end{align*}
  where it is assumed $\tau_1\leq 2$, and $\tcO$ hides constants and $\ln(nmd)$.
  Choose step size $\eta := 4/\rho^2$, and run gradient descent
  for $t:=1/(8\epsopt)$ iterations, selecting iterate
  $W_{\leq t} :=\argmin\{\hcR(W_i) : i \leq t, \|W_i-W_0\|\leq \radopt\}$.
Then, with probability at least $1-25\delta$,
  \begin{align*}
& \cR(W_{\leq t}) - \barcR
    &\text{(logistic error)}\phantom{.}
    \\
  \leq\qquad& \klb(p_y, \phi_{\infty}) + \del[1]{e^{\tau_1 + \tau_0} - 1}\cR(\barUi)
            &\text{(reference model error)}\phantom{.}
            \\
  + \quad & e^{\tau_1} R^2 \epsopt
          &\text{(optimization error)}\phantom{.}
          \\
  + \quad &
e^{\tau_1} (\rho B + R)\tau_n
          & \hspace{2em}\text{(generalization error)},\\
          \intertext{where the classification and calibration errors satisfy}
& {}\cR(W_{\leq t}) - \barcR
    &\text{(logistic error)}\phantom{.}
    \\
            \geq \qquad& 2\int \del{\phi(f(x;W_{\leq t})) - p_y}^2\dif\mu_x(x)
               &\text{(calibration error)}\phantom{.}
               \\
   \geq \qquad &\frac 1 2 \del{ \cRz(W_{\leq t}) - \barcRz}^2
                &\text{(classification error)}.
  \end{align*}
  Lastly, for any $\epsilon > 0$, there exists $\barUi^{(\eps)}$
  with $\sup_v \|\barUi^{(\eps)}(v)\|<\infty$
  and whose conditional model $\phi_{\infty}^{(\eps)}(x) := \phi(f((x,1)/\sqrt{2};\barUi^{(\eps)}))$
  satisfies $\klb(p_y, \phi_{\infty}^{(\eps)}) \leq \epsilon$.
\end{theorem}

\begin{remark}\label{rem:main}
  The key properties of \Cref{fact:main} are as follows.

  \begin{enumerate}

    \item
      \textbf{(Achieving error $\cO(\eps)$ in three different regimes.)}
      As \Cref{fact:main} is quite complicated, consider three different situations,
      which vary the reference model $\barUi$ and its norm upper bound
      $R := \max\{4, \rho, \sup_v\|\barUi(v)\|\}<\infty$,
      as well as the early stopping radius $\radopt$.  Let target population (excess)
      risk $\eps>0$
      be given, set $\epsopt = \eps$ and $t = 1/(8\epsopt)$ as in \Cref{fact:main},
      and suppose $n \geq 1/\eps^2$ samples: in each of the
      three following settings,
      the other parameters parameters (namely $\rho$ and $m$) will be chosen
      to ensure a final error $\cR(W_{\leq t}) - \barcR = \cO(\eps)$.

      \begin{enumerate}
        \item
          \textbf{(Easy data.)}  Suppose a setting with \emph{easy data}:
          specifically, suppose that for chosen target accuracy $\eps>0$,
          there exists $\barUi$ with
          $\klb(p_y, \phi_\infty) = \cR(\barUi) - \barcR \leq \cR(\barUi) \leq \eps$.
          If we set $\rho = 1$ and $m \geq R^8$, then $(\tau_n,\tau_1,\tau_0)$ are all
          constant, and we get a final bound
          $\cR(W_{\leq t}) - \barcR = \cO(\eps)$.

          Note crucially that $m \approx R^8$ sufficed for this setting; this was
          a goal of the present analysis, as it recovers the \emph{polylogarithmic width}
          analyses from prior work \citep{ziwei_ntk,gu_polylog}.
          Those works however either used a separation condition due to
          \citet{nitanda_refined} in the shallow case, or an assumption on the approximation
          properties of the sampled weights (a random variable) in the deep case, and thus the
          present analysis provides not just a re-proof, but a simplification and generalization.
          This was the motivation for the strange \emph{multiplicative} form of the errors
          in \Cref{fact:main}: had we used the more common additive errors with standard
          linearization tools, a polylogarithmic
          width proof would fail.

        \item
          \textbf{(General data, clairvoyant early stopping radius $\radopt$.)}
          Suppose that we are in the general noisy case, meaning
          any $\barUi$ we pick has a large error
          $\klb(p_y, \phi_\infty)$, but we magically know the
          $R$ corresponding to a good $\barUi$, and can choose $\radopt = R/\rho$.
          Unlike the previous case, to achieve some target error $\eps$,
          we need to work harder to control the term $\sbr{\exp(\tau_1+\tau_0) - 1}\cR(\barUi)$,
          since we no longer have small $\cR(\barUi)$; to this end,
          since $\tau_1 =\tcO(R^{4/3} / (m \rho^2)^{1/6})$ and $\tau_0 = \tcO(\rho + 1/m^{1/4})$,
          choosing $\rho = m^{-1/8}$ and $m = 1/\eps^{8}$ gives
          $\tau_1=\tcO(\eps)$ and $\tau_0 = \tcO(\eps)$, and together
          $\cR(W_{\leq t}) - \barcR = \cO(\eps)$.

        \item
          \textbf{(General data, worst-case early stopping.)}
          Suppose again the case of general noisy data
          with large error
          $\klb(p_y, \phi_\infty)$ for any $\barUi$ we pick,
          but now suppose we have
          no early stopping hint,
          and pessimistically set $\radopt = \infty$.
          As a consequence of all of this,
          the term $B$ can scale as $t^{2/3}/\rho= 1 / (\rho \eps^{2/3})$,
          thus to control $\tau_1 = \tcO( (1/\eps)^{2/3} / (m\rho^2)^{1/6})$ and 
          $\tau_0 = \tcO(\rho + 1/m^{1/4})$,
          we can again choose $\rho = m^{-1/8}$, but need a larger width
          $m = 1/\eps^{40/3}$.  Together, we once again achieve population excess risk
          $\cR(W_{\leq t}) - \barcR = \cO(\eps)$.
      \end{enumerate}

      Summarizing, a first key point is that arbitrarily small excess risk $\cO(\eps)$
      is always possible; as discussed, this is in contrast to prior work,
      which either only gave training error guarantees, or required restrictive conditions for
      small test error.  A second key point is that the parameters of the bound, most notably
      the required width, will shrink greatly when either the data is easy, or an optimal
      stopping radius $\radopt$ is known.

    \item
      \textbf{(Consistency.)}
      \emph{Consistency} is a classical statistical goal of achieving the optimal test error
      almost surely
      \emph{over all possible predictors} as $n\to\infty$; here it is
      proved as a consequence of \Cref{fact:main}, namely the preceding argument
      that we can achieve excess risk $\cO(\eps)$ even with general prediction problems
      and no early stopping hints ($\radopt=\infty$).
      The consistency guarantee is stated formally in \Cref{fact:consistency}.
      The statement takes the width to infinity, and demonstrates another advantage of using 
      an infinite-width reference model: within the proof, after fixing a target accuracy,
      the reference model is \emph{fixed} and used for all widths simultaneously.

\begin{figure}[b!]
\begin{tcolorbox}[enhanced jigsaw, empty, sharp corners, colback=white,borderline north = {1pt}{0pt}{black},borderline south = {1pt}{0pt}{black},left=0pt,right=0pt,boxsep=0pt,rightrule=0pt,leftrule=0pt]
\centering
    \includegraphics[width=0.7\textwidth]{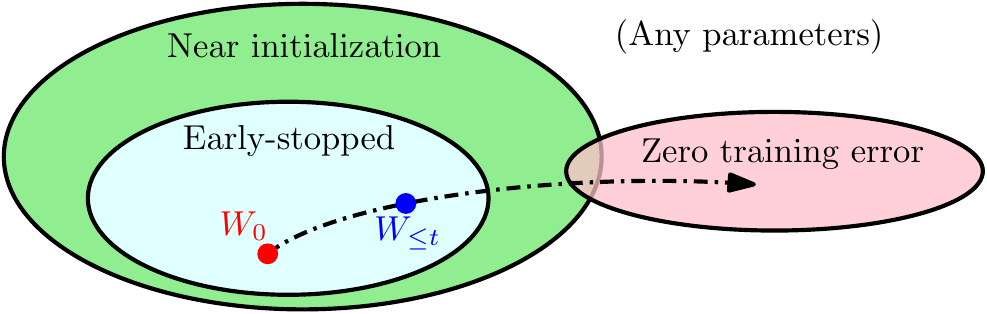}
    \caption{The setting of this paper, contrasted with standard settings.  \Cref{fact:main}
      considers iterate $W_{\leq t}$, which is somewhere in the \emph{early-stopped} ball
      around the initial random choice $W_0$.
      This early-stopped ball is well inside the \emph{near initialization} or \emph{NTK}
      ball, since in noisy settings, the early-stopped ball will not reach zero training error,
      whereas the NTK ball will.  Meanwhile, the NTK itself requires early stopping and is a
      subset of the space of all parameters.
    }
    \label{fig:venn}
   \end{tcolorbox}
\end{figure}
    \item
      \textbf{(Non-vacuous generalization, and an estimate of $R$.)}
      There is extensive concern throughout the community that generalization estimates
      are hopelessly loose
      \citep{nati_implicit_gen,zhang_gen,roy_vacuous}; to reduce the concern here,
      we raise two points.  Firstly, these concerns usually involve explicit calculations of
      generalization bounds which have terms scaling with some combination
      of $\|W\|$ (not $\|W - W_0\|$)
      and $m$; e.g,. one standard bound has spectral norms $\|W\|_2$
      and $(2,1)$ matrix norms $\|(W - W_0)^\T\|_{2,1}$, which are upper bounded by
      $\|W - W_0\|\sqrt{m}$ \citep{spec}.
      By contrast, the present work uses a new generalization bound
      technique
      (cf. \Cref{fact:shallow:gen:frob:1})
      which first \emph{de-linearizes} the network, then applies a \emph{linear
      generalization bound} which has only $\|W-W_0\|$ and no explicit $\poly(m)$, and then
      \emph{re-linearizes}.

      Secondly, there may still be concern that the story here is broken due to the term $R$,
      and namely the non-existence of good choices for $\barUi$.  For this, we conducted a simple
      experiment.  Noting that we can freeze the initial features and train linear predictors of
      the form $f^{(0)}(x;V)$ for weights $V\in\R^{m\times d}$ (cf. \cref{sec:notation}),
      and that the performance converges to the infinite-width performance as $m\to\infty$,
      we fixed a large width and trained two prediction tasks: an \emph{easy} task
      of MNIST 1 vs 5 until $R_{\text{easy}}/\sqrt{n} \approx 1/2$, and a \emph{hard} task
      of MNIST 3 vs 5 until $R_{\text{hard}}/\sqrt{n} \approx 1/2$.
      After training, we obtained test error $\cR(V_{\text{easy}}) \approx 0.01$
      and $\cR(V_{\text{hard}})\approx 0.08$.  Plugging all of these terms back in to the bound,
      firstly these techniques can yield a non-vacuous generalization bound, secondly they
      do not exhibit bad scaling with large width, and thirdly they do reflect the difficulty
      of the problem, as desired.

    \item
      \textbf{(Early stopping and the NTK.)}
      As discussed above, when the data is noisy, the method is explicitly early stopped,
      either by clairvoyantly choosing $\radopt$, or by making $t$ small.  In this setting,
      the optimization accuracy $\epsopt$ is an \emph{excess} empirical risk, meaning in particular
      that $0$ training error (the \emph{interpolation regime} \citep{double_descent})
      will \emph{not} be reached.  This is in stark contrast to standard
      NTK analyses \citep{allen_deep_opt}, which guarantee zero training error, but can not
      ensure good test error in general.  Since the NTK itself is an early stopping (as in,
      if one continues to optimizes, one exits the NTK), then the early stopping in this work
      is even earlier than the NTK early stopping; this situation is summarized in
      \Cref{fig:venn},
      and will be revisited for the lower bound in \Cref{sec:intro:lb}.

    \item
      \textbf{(Classification and calibration.)}
      The relationship to classification and calibration errors is merely a restatement
      of existing results \citep{zhang_convex_consistency,bartlett_jordan_mcauliffe},
      though it is reproved here in an elementary way for the special case of the logistic loss.
      Similarly, the guarantee that
      $\klb(p_y, \phi_{\infty}^{(\eps)})$ can be made arbitrarily small is also not a primary
      contribution, and indeed most of the heavy lifting is provided both by prior work
      in neural network approximation \citep{barron_nn}, and by the
      existing and reliable machinery
      for proving consistency \citep{schapire_freund_book_final}.
      As such, the consistency result is stated only much later in \Cref{fact:consistency},
      and our focus is on the exact risk guarantees in \Cref{fact:main}.

    \item
      \textbf{(Inputs with bias: $(x,1)/\sqrt{2}\in\R^{d+1}$.)}
      The end of \Cref{fact:main} appends a constant to the input (and rescales),
      which simulates a bias
      term inside each ReLU; this is necessary since our models 
      are (sigmoid mappings of) homogeneous functions, whereas $p_y$ is general.
Biases are also simulated in this way in the consistency result in \Cref{fact:consistency}.
      \qedhere
  \end{enumerate}
\end{remark}

Further discussion of \Cref{fact:main}, including the formal consistency result
(cf. \Cref{fact:consistency}) and a proof sketch, all appear in \Cref{sec:gd}.
Full proofs appear in the appendices.

\subsection{Should we early stop?}
\label{sec:intro:lb}

\begin{figure}[b!]
\begin{tcolorbox}[enhanced jigsaw, empty, sharp corners, colback=white,borderline north = {1pt}{0pt}{black},borderline south = {1pt}{0pt}{black},left=0pt,right=0pt,boxsep=0pt,rightrule=0pt,leftrule=0pt]
\centering
    \begin{subfigure}[t]{0.46\textwidth}
      \centering
      \includegraphics[width=\textwidth]{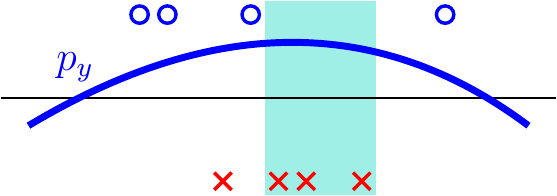}
      \caption{Conditional model $p_y$ and some noisy data.  A smoothed prediction rule
      would perform well.}
      \label{fig:interp:1}
    \end{subfigure}\hfill
    \begin{subfigure}[t]{0.46\textwidth}
      \centering
      \includegraphics[width=\textwidth]{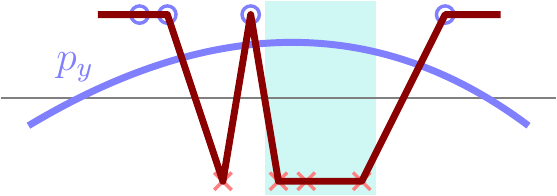}
      \caption{A \emph{local interpolation rule} working very hard to fit the noisy data.}
      \label{fig:interp:2}
    \end{subfigure}\caption{When data is noisy, it's best to give up on a few points.  The shaded region here
      highlights
      consecutive points with the wrong label; as in \Cref{fact:lb:local},
      prediction rules that locally interpolate will have a large population risk in these
      regions.
    }
    \label{fig:interp}
   \end{tcolorbox}
\end{figure}

\Cref{fact:main} uses early stopping: it can blow up if
$\barcR > 0$ and the two gradient descent parameters $\radopt$ and $1/\epsopt$
are taken to $\infty$ in an uncoordinated fashion.
Part of this is purely technical: as with many neural network optimization proofs,
the analysis breaks when far from initialization.
It is of course natural to wonder what happens
if one trains indefinitely, entering the actively-studied \emph{interpolation regime}
\citep{daniel_interpolation,double_descent,peter_benign}. Furthermore, there is evidence
that gradient descent on shallow networks limits towards a particular interpolating
choice, one with large margins
\citep{nati_logistic,riskparam_logreg,kaifeng_jian_margin,chizat_bach_imp,ziwei_directional}.
Is this behavior favorable?

While we do not rule out that the interpolating solutions found by neural networks
perform well, we show that at least in the low-dimensional (univariate!) setting,
if a prediction
rule perfectly labels the data and is not too wild between training points, then it is guaranteed
to achieve poor test loss on noisy problems.
This negative observation is not completely at odds with the interpolation literature,
where the performance of some rules improves with dimension \citep{daniel_interpolation}.

\begin{proposition}
  \label{fact:lb:local}
  Given a finite sample $((x_i,y_i))_{i=1}^n$ with $x_i\in\R$ and $y_i\in\{\pm 1\}$,
  let $\cF_n$ denote the collection of \emph{local interpolation rules}
  (cf. \Cref{fig:interp}):
  letting $x_{(i)}$ index examples in sorted order, meaning
  $x_{(1)} \leq x_{(2)}\leq \cdots \leq x_{(n)}$,
  define $\cF_n$ as
  \begin{align*}
    \cF_n
    := \big\{ f : \R \to \R \ : \ {}
    &\forall i\ {}f(x_{(i)}) = y_{(i)}, \text{ and}
      \\
    &\text{if } y_{(i)} = y_{(i+1)}, \text{ then }
    \inf_{\alpha\in[0,1]} f\del[1]{ \alpha x_{(i)} + (1-\alpha)x_{(i+1)} }y_{(i)} > 0
    \big\}.
  \end{align*}
  Then there exists a constant $c>0$ so that with probability at least $1-\delta$ over the draw of $((x_i,y_i))_{i=1}^n$
  with $n \geq \ln(1/\delta)/c$,
  every $f\in\cF_n$ satisfies $\cRz(f) \geq \bar\cRz(f) + c$.
\end{proposition}

Although a minor contribution, this result will be discussed briefly
in \Cref{sec:lb:local}, with detailed proofs appearing in the appendices.
For a similar discussion for nearest neighbor classifiers albeit under a few additional
assumptions, see \citep{nakkiran_bansal}.

\subsection{Related work}

\paragraph{Analyses of gradient descent.}
The proof here shares the most elements with recent works whose width could be polylogarithmic
in the sample size and desired target accuracy $1/\eps$
\citep{ziwei_ntk,gu_polylog}.  Similarities include using a regret inequality as the core
of the proof, using an infinite-width target network \citep{nitanda_refined,ziwei_ntk},
and using a \emph{linearization inequality} \citep{gu_polylog,allen_deep_opt}.
On the technical side, the present work differs in the detailed treatment of the logistic loss,
and in the linearization inequality which is extended to hold over the population risk;
otherwise, the core gradient descent analysis here is arguably simplified relative
to these prior works.
It should be noted that the use of a regret inequality here and in the previous works crucially
makes use of a negated term which was dropped in some classical treatments;
this trick is now re-appearing
in many places \citep{orabona2021parameterfree,frei_agnostic}.

There are many other, somewhat less similar works in the vast literature of gradient descent on neural networks,
in particular in the neural tangent regime
\citep{jacot_ntk,li_liang_nips,du_iclr}.
These works often handle not only training error, but also testing error
\citep{allen_3_gen,arora_2_gen,cao_deep_gen,nitanda_refined,ziwei_ntk,gu_polylog}.
As was mentioned before, these works do not appear to handle arbitrary target
models; see for instance the modeling discussion in \citep[Section 6]{arora_2_gen}.
As another interesting recent example,
some works explicitly handle certain noisy conditional models,
but with error terms that do not go to zero in general
\citep{liang2021achieving}.

\paragraph{Consistency.}
Consistency of deep networks with classification loss and
\emph{some} training procedure is classical; e.g.,
in \citep{lugosi_consistency}, the authors show that it suffices to
run a computationally intractable algorithm on an architecture chosen to balance VC dimension
and universal approximation.  Similarly, the work here makes use of Barron's superposition
analysis in an infinite-width form to meet the Bayes risk
\citep{barron_nn,ntk_apx}.
The statistics literature has many other works giving
beautiful analyses of neural networks, e.g., even with minimax rates
\citep{schmidthieber}, though it appears this literature generally does not consider gradient
descent and arbitrary classification objectives.

In the boosting literature, most consistency proofs only consider classification loss
\citep{bartlett_traskin_adaboost,schapire_freund_book_final},
though there is a notable exception which controls the convex loss (and thus calibration),
although the algorithm has a number of modifications \citep{zhang_yu_boosting}.
In all these works, arbitrary $p_y$ are not handled explicitly as here, but rather
\emph{implicitly} via assumptions on the expressiveness of the weak learners.
One exception is the logistic loss boosting proof of \citet{mjt_log_cons},
which explicitly handles measurable $p_y$ via Lusin's theorem as is done here,
but ultimately the proof only controls classification loss.

Following the arXiv posting of this work, a few closely related works appeared.
Firstly, \citet{ilja_consistency} show that the expected excess risk can scale with
$\|W_t-W_0\|_\tF / n^\alpha$, though in contrast with the present work, it is not shown
that this ratio can go to zero for arbitrary prediction problems, and moreover the
bound is in expectation only.  Secondly, the work of \citet{smoking_gun} is even closer,
however it requires a condition on the Fourier spectrum of the conditional model $p_y$,
which is circumvented here
via a more careful Fourier analysis due to \citet{ntk_apx}.

\paragraph{Calibration.}
There is an increasing body of work considering the (in)ability of networks trained
with the logistic loss to recover the underlying conditional model.
Both on the empirical side \citep{guo2017calibration} and on the theoretical side
\citep{bai2021dont}, the evidence is on the side of the logistic loss doing poorly, specifically
being \emph{overconfident}, meaning the sigmoid outputs are too close to $0$ or $1$.
This overconfident regime corresponds to large margins; indeed,
since gradient descent can be proved in some settings
to exhibit unboundedly large unnormalized margins on
all training points \citep{kaifeng_jian_margin}, the sigmoid mapping of the
predictions will necessarily limit to exactly $0$ or $1$.
On the other hand, as mentioned in \citep{bai2021dont}, regularization suffices
to circumvent this issue. In the present work, a combination of early stopping
and small temperature are employed.
As mentioned before, calibration is proved here as an immediate corollary of meeting
the optimal logistic risk via classification calibration
\citep{zhang_convex_consistency,bartlett_jordan_mcauliffe}.

\subsection{Further notation and technical background}
\label{sec:notation}

The loss $\ell$, risks $\cR$ and $\hcR$, and network $f$ have been defined.
The misclassification risk $\cRz(f) = \Pr[\sgn(f(X))\neq Y]$ appeared in \Cref{fact:main},
where $\sgn(f(x)) = 2\cdot \1[f(x)\geq 0] - 1$.

Next, consider the ``gradient'' of $f$ with respect to weights $W$:
\[
  \nabla f(x;W) := \frac {\rho}{\sqrt m}\sum_{j=1}^m a_j \1[w_j^\T x \geq 0] \ve_jx^\T;
\]
it may seem the nondifferentiability at $0$ is concerning, but in analyses close to initialization
(as is the one here), few activations change, and their behavior is treated in a worst-case
fashion.  Note that, as is easily checked with this expression,
$\|\nabla f(W)\| \leq \rho$, which is convenient in many places in the proofs.
Here $\|\cdot\|$ denotes the Frobenius norm; $\|\cdot\|_2$ will denote the spectral norm.

Given weight matrix $W_i$ at time $i$, let $(w_{i,j}^\T)_{j=1}^m$ refer to its rows.
Define features $\ffi$ at time $i$ and a corresponding empirical risk $\hcRi$ using
the features at time $i$ as
\begin{align*}
  \ffi(x;V)
  &:= \ip{\nf(x;W_i)}{V} = \frac \rho {\sqrt m} \sum_j a_j v_j^\T x \1[w_{i,j}^\T x\geq 0],
  \\
  \hcRi(x;V) &:= \hcR(x\mapsto \ffi(x;V)).
\end{align*}
By $1$-homogeneity of the ReLU, $\ffi(x;W_i) = f(x;W_i)$, which will also be used often.
These features at time $i$, meaning $\ffi$ and $\hcRi$,
are very useful in analyses near initialization, as they do not change much.
As such, $\ff{0}$ and $\cR^{(0)}$ and $\hcR^{(0)}$ will all appear often as well.

To be a bit pedantic about the measure $\mu$: as before, there is a joint distribution $\mu$,
which is over the Borel $\sigma$-algebra on $\R^d\times\{\pm 1\}$, where $\|x\|\leq 1$ almost
surely.  This condition suffices to grant both a \emph{disintegration} of $\mu$ into
marginal $\mu_x$ and conditional $p_y$ \citep[Chapter 6]{kallenberg2002foundations},
and also Lusin's theorem \citep[Theorem 7.10]{folland}, which is used to switch from a
measurable function to a continuous one in the consistency proof (cf. \Cref{fact:consistency}).

\section{Discussion and proof sketch of \Cref{fact:main}}
\label{sec:gd}

This section breaks down the proof and discussion into four subsections: a section
with common technical tools,
then sections for the analysis of generalization, optimization, and approximation.

\subsection{Key technical lemmas}

There are two main new technical ideas which power many parts of the proofs: a multiplicative
error property of the logistic loss, and a \emph{linearization over the sphere}.

The logistic loss property is simple enough: for any $a\geq b$, it holds that
$\ell(-a)/\ell(-b)\leq \exp(a-b)$.  On the surface, this seems innocuous, but this simple
inequality allows us to reprove existing polylogarithmic width results for easy data
\citep{ziwei_ntk,gu_polylog}, however making use of a proof scheme which is slightly more
standard, or at the very least more apparently a smooth convex proof with just this one
special property of the logistic loss (as opposed to a few special properties).

The second tool is more technical, and is used crucially in many places in the proof.
Many prior analyses near initialization bound the quantity
\[
  f(x;V) - f(x;W) - \ip{\nf(x;W)}{V-W},
\]
where $V$ and $W$ are both close to initialization
\citep{allen_deep_opt,cao_deep_gen,gu_polylog}.  These proofs are typically performed on
a fixed example $x_k$, and then a union bound carries them over to the whole training set.
Here, instead, such a bound is extended to hold \emph{over the entire sphere}, as follows.

\begin{lemma}[Simplification of \Cref{fact:shallow:linearization:frob}]\label{fact:shallow:linearization:frob:simplified}
  Let scalars $\delta>0$ and $R_V\geq 1$ and $R_B\geq 0$ be given.
  \begin{enumerate}

    \item
With probability at least $1-3n\delta$,
      \[
\sup_{\substack{\|W_i - W_0\| \leq R_V\\
            \|W_j - W_0\| \leq R_V\\
        \|B - W_0\|\leq R_B}}
        \frac {\hcR^{(i)}(B)}{\hcR^{(j)}(B)}
        \leq
\exp\del{
          \frac {6\rho \del{R_B + 2R_V}R_V^{1/3}\ln(e/\delta)^{1/4}}{m^{1/6}}
        }
        .
      \]

    \item
Suppose $m \geq \ln(edm)$.
      With probability at least $1- (1 + 3 (d^2m)^{d})\delta$, 
      \[
        \sup_{\|W_i - W_0\| \leq R_V}
        \frac {\cR(W_i)}{\cR^{(0)}(W_i)}
        \leq
\exp\del{
          \frac {25 \rho R_V^{4/3} \sqrt{\ln(edm/\delta)}}{m^{1/6}}
        }
        .
      \]
  \end{enumerate}
\end{lemma}

The preceding \namecref{fact:shallow:linearization:frob} combines both the linearization
technique and the multiplicative error property: it bounds how much the empirical and true risk
change for a fix weight matrix if we swap in and out the features at different iterations.
That these bounds are a ratio is due to the multiplicative error property.  That the second
part holds over the true risk, in particular controlling behavior over all $\|x\|\leq 1$,
is a consequence of the new more powerful linearization technique.
This linearization over the sphere is used crucially in three separate places:
we use it when controlling the range in the generalization proofs,
when \emph{de-linearizing} after generalization,
and when sampling from the infinite-width model $\barUi$.
The method of proof is
inspired by the concept of \emph{co-VC dimension} \citep{gurvits_koiran}:
the desired inequality is first union bounded over a cover of the sphere,
and then relaxed to all points on the sphere.
A key difficulty here is the non-smoothness of the ReLU, and a key lemma
establishes a smoothness-like inequality (cf. \Cref{fact:covc:helper}).
These techniques appear in full in the appendices.

\subsection{Generalization analysis}

The generalization statement appears as \Cref{fact:shallow:gen:frob:1} in the appendices,
together with its proofs,
but here is a sketch of the key elements.
To start, rather than directly studying uniform convergence properties of the networks
reachable by gradient descent,
\Cref{fact:shallow:linearization:frob:simplified}
is applied over the training set to convert the network to a linear predictor,
and only then is generalization \emph{of linear predictors} applied;
this use of generalization for linear predictors and not of general networks is how the bound
pays only logarithmically in the width, and otherwise has just a Frobenius norm dependence on
the weight matrices (minus initialization), which is in contrast with standard generalization
bounds.
Thereafter, \Cref{fact:shallow:linearization:frob:simplified} is applied once more,
but on the population risk (which
uses the approximation guarantee over the entire sphere and not just the training set), which \emph{de-linearizes} the
linear predictor used for generalization and gives a test error guarantee for the original network.

Typically the easiest step in proving generalization is to provide a
worst-case estimate on the range of the predictor, however a standard worst-case estimate
in this setting
incurs a polynomial dependence on network width.
To avoid this, we once again use the tools of
\Cref{fact:shallow:linearization:frob:simplified} to control the range with high
probability.

\subsection{Gradient descent analysis}

A common tool in linear prediction is the regret inequality
\[
  \|v_t - z\|^2 + 2\eta\sum_{i<t} \hcR(v_{i+1}) \leq \|v_0 - z\|^2 + 2t\eta \hcR(z),
\]
which can be derived by expanding the square in $\|v_t - z\|^2$ and applying
smoothness and convexity.  The term $\|v_t-z\|^2$ is often dropped, but can be used in a very
convenient way:  by the triangle inequality, if $\|v_t-v_0\|\geq 2\|z-v_0\|$, then the norm
terms above
may be canceled from both sides, which leaves only the empirical risk terms; overall, this
argument ensures both small norm and small empirical risk.
This idea has appeared in a variety of works
\citep{shamir2020gradient,gd_reg}, and is used here to provide a convenient norm control,
allowing linearization and all other proof parts to go through.  Combining this idea
with the earlier generalization analysis and a few other minor tricks
gives the following bounds, which in turn provide most of \Cref{fact:main}.

\begin{lemma}\label{fact:shallow:magic:linearized}
  Let
  temperature $\rho > 0$,
  step size $\eta \leq 4/\rho^2$,
  optimization accuracy $\epsopt > 0$,
  radius $\radopt>0$,
  network width $m\geq \ln(emd)$,
  reference matrix $Z\in\R^{m\times d}$,
  corresponding scalar $R_Z\leq\radopt$ where $R_Z \geq \max\{1, \eta \rho, \|W_0 - Z\|\}$,
  and $t \geq 1 / (2\eta\rho^2 \epsopt)$ be given;
  correspondingly define $W_{\leq t} := \argmin\{ \hcR(W_i) : i \leq t, \|W_i-W_0\|\leq \radopt\}$.
  Define effective radius
  $B := \min\cbr[1]{\radopt,\ {}3R_Z + 2e \sqrt{\eta t \hcR^{(0)}(Z)} }$,
  and linearization and generalization errors
  \[
\tau := \frac {25 \rho B^{4/3}\sqrt{d\ln(em^2d^3/\delta)}}{m^{1/6}},
    \qquad
    \tau_n := \frac {80\del{d \ln(em^2 d^3/\delta)}^{3/2}}{\sqrt n},
  \]
  and suppose $\tau \leq 2$.
  Then, with probability at least $1-3n\delta$, 
  the selected iterate $W_{\leq t}$
  satisfies $\|W_{\leq t}-W_0\|\leq B$,
  along with the empirical risk guarantee
  \begin{align*}
    \hcR(W_{\leq t}) \leq e^{2\tau} \hcR^{(0)}(Z) + e^{\tau} (\rho R_Z)^2 \epsopt,
  \end{align*}
  and by discarding an additional $16\delta$ failure probability,
  then $\hcR^{(0)}(Z) \leq \cR^{(0)}(Z) + \rho R_Z \tau_n$,
  and
  \begin{align*}
    \cR(W_{\leq t})
    &\leq
    e^{4\tau} \cR^{(0)}(Z)
    + e^{3\tau} (\rho R_Z)^2 \epsopt
    +
    e^{4\tau} (B+R_Z)\rho \tau_n.
\end{align*}
\end{lemma}

This version of the statement, unlike \Cref{fact:main}, features an arbitrary reference
\emph{matrix} $Z$.  This is powerful, though it can be awkward, since $W_0$ is a random variable.

\subsection{Approximation analysis, consistency, and the proof of \Cref{fact:main}}

Rather than trying to reason about good predictors which may happen to be close to random
initialization,
the approach here is instead to start from deterministic predictors over the population
(e.g., $\barUi$),
and to use their structure to construct approximants
near the initial iterate, the random matrix $W_0$.
Specifically, the approach here is fairly brute force: given initial weights
$W_0$ with rows $(w_{0,j}^\T)_{j=1}^m$, the rows $(\baru_j)_{j=1}^m$
of the finite width reference matrix $\barU\in\R^{m\times d}$ intended to mimic $\barUi$
(which is after all a mapping $\barUi :\R^d \to \R^d$)
are simply
\begin{equation}
  \baru_j := \frac {a_j \barUi(w_{0,j})}{\rho\sqrt{m}} + w_{0,j}.
  \label{eq:barUi:sample}
\end{equation}
By construction, $\|\barU - W_0\|\leq R/\rho$, where $R:=\sup_v \|\barUi(v)\|$.
To argue that $\cR^{(0)}(\barU)$ and $\cR(\barUi)$ are close,
the risk control over the sphere in
\Cref{fact:shallow:linearization:frob:simplified}
is again used.
Plugging this $\barU$ into 
\Cref{fact:shallow:magic:linearized} and introducing $\klb(p_y, \phi_\infty)$
gives the first part of \Cref{fact:main}, and the second part
of \Cref{fact:main} is from a few properties of the logistic loss summarized
in \Cref{fact:logistic:error}.

It remains to prove that for any $p_y$, there exists $\barUi$ with
$\phi(x\mapsto f((x,1)/\sqrt{2};\barUi)) \approx p_y$
(we must include a bias term, as mentioned in \Cref{rem:main}).
If $p_y$ were continuous, there is a variant
of \citeauthor{barron_nn}'s seminal universal approximation construction which explicitly
gives an infinite-width network of the desired form \citep{barron_nn,ntk_apx}.
To address continuity is even easier: \emph{Lusin's theorem} \citep[Theorem 7.10]{folland}
lets us take the measurable
function $p_y$, and obtain a continuous function that agrees with it on all but a negligible
fraction of the domain.  This completes the proof.

As mentioned, a key property of the reference model $\barUi$
is that it depends on neither the random sampling of data, nor the random sampling of weights.
This vastly simplifies the proof of consistency, where the proof scheme first fixes an $\eps>0$
and chooses a $\barUi$, and leaves it fixed as $m$ and $n$ vary.

\begin{corollary}\label{fact:consistency}
  Let early stopping parameter $\xi\in(0,1)$ be given,
  and for each sample size $n$, define a weight matrix $\widehat W_n \in \R^{m^{(n)}\times (d+1)}$
  and corresponding conditional probability model 
  $\widehat \phi_n(x) := \phi(f((x,1)/\sqrt 2;\widehat W_n))$ as follows.
  For each sample size $n$, let $(W^{(n)}_i)_{i\geq 0}$
  denote the corresponding sequence of gradient descent iterates obtained with parameter choices
$\rho^{(n)} := (m^{(n)})^{-1/8}$,
  and $m^{(n)} := n^{\frac{40}{3}(1-\xi)}$,
  and $\eta^{(n)} := 4/(\rho^{(n)})^2$,
  and $\epsopt^{(n)} := n^{\xi-1}$,
  and $t^{(n)} := n^{1-\xi}/8$,
  and choose the empirical risk minimizer over the sequence, meaning
  $\widehat W_n := \argmin\cbr[1]{\hcR(W^{(n)}_i) : i \leq t^{(n)} }$
  (in the notation of \Cref{fact:main}, this is $W_\leq t$ with $\radopt = \infty$).
  Then
  \[
    \cR(\widehat W_n) \xrightarrow{\phantom{\ {}L_2(\mu_x)}\ {}}  \barcR \text{ a.s.,}
    \qquad
    \quad
    \cRz(\widehat W_n) \xrightarrow{\phantom{\ {}L_2(\mu_x)\ {}}} \barcRz \text{ a.s.,}
    \qquad
    \quad
    \widehat \phi_n \xrightarrow{\ {}L_2(\mu_x)\ {}} p_y \text{ a.s.,}
  \]
  where the last convergence is in the $L_2(\mu_x)$ metric.
\end{corollary}

The use of a parameter $\xi\in(0,1)$ is standard in similar consistency results; see for instance
the analogous parameter in the consistency analysis of AdaBoost
\citep{bartlett_traskin_adaboost}.
Proofs, as usual, are in the appendices.

\section{Discussion and proof sketch of \Cref{fact:lb:local}}
\label{sec:lb:local}

\Cref{fact:lb:local} asserts that univariate \emph{local interpolation rules} --- 
predictors which perfectly fit the data, and are not too wild between data points of the
same label --- will necessarily achieve suboptimal population risk.
The proof idea seems simple enough: if the true conditional probability $p_y$
is not one of $\{0,\nicefrac 1 2,1\}$
everywhere, and is also continuous, then there must exist a region where it is well separated
from these three choices.
It seems natural that a constant fraction of the data in these regions
will form adjacent pairs with the wrong label; a local interpolation rule will fail on exactly
these adjacent noisy pairs,
which suffices to give the bound.  In reality, while this is indeed the proof scheme followed here,
the full proof must contend with many technicalities
and independence issues.  It appears in the appendices.

While the motivation in \Cref{sec:intro:lb} focused on neural networks which interpolate, and also
maximum margin solutions,
the behavior on this noisy univariate
data is also well-illustrated by $k$-nearest-neighbors classifiers ($k$-nn).
Specifically, $1$-nn is a local interpolant, and \Cref{fact:lb:local} applies.
On the other hand, choosing $k = \Theta(\ln(n))$ is known to provide
enough smoothing to achieve consistency and avoid interpolation \citep{DGL}.

It should be stressed again that even if the remaining pieces could be proved to apply this result
to neural networks, namely necessitating early stopping,
it would still be a univariate result only,
leaving open many interesting possibilities in higher dimensions.

\section{Concluding remarks and open problems}
\label{sec:open}

\paragraph{Empirical performance.}
Does the story here match experiments?  E.g., is it often the case that if a neural network
performs well, then so does a random feature model?  Do neural networks fail on noisy data if
care is not taken with temperature and early stopping?  Most specifically, is this part
of what happens in existing results reporting such failures \citep{guo2017calibration}?

\paragraph{Temperature parameter $\rho$.}
Another interesting point of study is the temperature parameter $\rho$.  It arises here
in a fairly technical way: if $p_y$ is often close to $1/2$,
then the random initialization of $W_0$ gets in the way of learning $p_y$.  The temperature
$\rho$ is in fact a brute-force method of suppressing this weight initialization noise.
On the other hand,
temperature parameters are common across many works which rely heavily on the detailed
real-valued outputs of sigmoid and softmax mappings; e.g., in the distillation literature
\citep{distillation}.  The temperature also plays the same role as the scale parameter
in the \emph{lazy training} regime
\citep{bach_chizat_note}.
Is $\rho$ generally useful, and does the analysis here relate to its practical utility?

\paragraph{Random features, and going beyond the NTK.}
The analysis here early stops before the feature learning begins to occur.
How do things fare outside the NTK?  Is there an analog of \Cref{fact:main},
still stopping shy of the interpolation pitfalls of \Cref{fact:lb:local},
but managing to beat random features with some generality?

\paragraph{The logistic loss.}
One reason the logistic is used here is its simple interplay with calibration
(e.g., see the elementary proof of \Cref{fact:logistic:error}, as compared with the full
machinery of classification calibration
\citep{zhang_convex_consistency,bartlett_jordan_mcauliffe}).
The other key reason was the multiplicative error property \Cref{fact:logistic:error}.
Certainly, the logistic loss is widely used in practice; are the preceding technical points
at all related to the widespread empirical use of the logistic loss?

\subsubsection*{Acknowledgments}
The authors are grateful for support from the NSF under grant IIS-1750051.
MT thanks many friends for illuminating and motivating discussions:
Daniel Hsu,
Phil Long,
Maxim Raginsky,
Fanny Yang.

\bibliography{bib}
\bibliographystyle{plainnat}

\clearpage

\appendix

\section{Proof of \Cref{fact:main} and supporting results}

This appendix section proves all bounds necessary for \Cref{fact:main},
and also proves the consistency statement in \Cref{fact:consistency}.

\subsection{Technical preliminaries}

First, the key logistic loss properties.

\begin{lemma}\label{fact:logistic:error}
  \begin{enumerate}
    \item
      For any $a\geq b$,
      \[
        \frac{\phi(a)}{\phi(b)} \leq e^{a-b}
        \qquad
        \textup{and}\qquad
        \frac {\ell(-a)}{\ell(-b)} \leq e^{a-b}.
      \]
      In particular, for any $f,g$ with $\sup_{\|x\|\leq 1} |f(x)-g(x)| \leq \tau$,
      \[
        e^{-\tau} \cR(f) \leq \cR(g) \leq e^\tau \cR(f).
      \]
If only $\max_k |f(x_k) - g(x_k)| \leq \tau$,
      then
      $e^{-\tau} \hcR(f) \leq \hcR(g) \leq e^\tau \hcR(f)$.

    \item
For any $f:\R^d\to\R$ and corresponding conditional model $\phi_f(x) := \phi(f(x))$,
      \[
        \frac 1 2 \del{ \cRz(f) - \barcRz}^2
        \leq 2\int (\phi_f(x) - p_y(x))^2\dif\mu_x(x)
        \leq \klb(p_y, \phi_f)
        = \cR(f) - \barcR.
      \]

  \end{enumerate}
\end{lemma}

\begin{proof}\begin{enumerate}
    \item
      Since $a \geq b$, then $e^{b-a} \leq 1$, and
      \[
        \frac {\phi(a)}{\phi(b)}
        = \frac {1 + e^{-b}}{1 + e^{-a}}
        = e^{a-b} \del{\frac {e^{b-a} + e^{-a}}{1 + e^{-a}}}
        \leq e^{a-b},
      \]
      whereby
      \[
        \int_{-\infty}^{a} \phi(r)\dif r
        = \int_{-\infty}^{b} \phi(r + (a-b))\dif r
        \leq e^{a-b} \int_{-\infty}^{b} \phi(r)\dif r.
      \]
      Consequently,
      \begin{align*}
        \ell(-a) = - \int_{-a}^\infty \ell'(r) \dif r = \int_{-a}^\infty \phi(-r)\dif r
        = \int_{-\infty}^{a} \phi(r)\dif r
        \leq e^{a-b} \int_{-\infty}^{b} \phi(r)\dif r
        = e^{a-b} \ell(-b).
      \end{align*}
      The first set of claims for risk follow from the fact that for any pair $(x,y)$
      and $\tau \geq 0$,
      \[
        \ell(yr + y^2 \tau) \leq \ell(yr) \leq \ell(yr - y^2 \tau),
      \]
      whereby
      \[
        \cR(f) = \bbE \ell(yf(x)) \leq \bbE \ell(yg(x) - \tau) 
        \leq e^\tau \bbE \ell(yg(x))  = e^\tau \cR(g).
      \]
      The proof for empirical risk is similar, but only relies upon behavior on
      the finite sample.

    \item
      From standard results in the literature on classification calibration
      \citep{zhang_convex_consistency,bartlett_jordan_mcauliffe},
      the optimal logistic loss pointwise satisfies
      \[
        \bar r_x := \inf_{r\in \R} p_y(x) \ell(r) + (1-p_y(x))\ell(-r) = -p_y(x) \ln p_y(x) - (1-p_y(x))\ln(1-p_y(x)).
      \]
      Consequently, for any predictor $f:\R^d \to \R$ and corresponding probability model
      $\phi_f(x) := \phi(f(x))$, note that
      \begin{align*}
        \cR(f)
        &= \int \del{p_y(x) \ln(1+ \exp(-f(x))) + (1-p_y(x)) \ln(1+\exp(f(x)))}\dif\mu_x(x)
        \\
        &= \int \del{ - p_y(x) \ln \phi_f(x) - (1-p_y(x)) \ln(1-\phi_f(x))}\dif\mu_x(x),
      \end{align*}
      and thus
      \[
        \cR(f) - \barcR = \klb(p_y, \phi_f).
      \]
      By Pinsker's inequality,
      \begin{align*}
        \klb(p_y, \phi_f)
    &= \int \del{ p_y(x) \ln \frac {p_y(x)}{\phi_f(x)} + (1-p_y(x)) \ln \frac {1-p_y(x)}{1-\phi_f(x)} }\dif \mu_x(x)
    \\
    &\geq \frac 1 2 \int \del{ |p_y(x) - \phi_f(x)| + | (1-p_y(x)) + (1-\phi_f(x))| }^2 \dif \mu_x(x)
    \\
    &= 2 \int \del{ p_y(x) - \phi_f(x) }^2 \dif \mu_x(x).
      \end{align*}
      If $\sgn(\phi_f(x)-1/2) \neq \sgn(p_y(x)-1/2)$,
      then $|\phi_f(x) - p_y(x)| \geq |p_y(x) - 1/2|$, and so
      \begin{align*}
        \cRz(f) - \barcRz
    &= \int \1[\sgn(\phi_f(x) - 1/2) \neq \sgn(p_y(x)-1/2)]\cdot |2p_y(x)-1|\dif\mu_x(x)
    \\
    &\leq 2 \int |\phi_f(x) - p_y(x)|\dif\mu_x(x)
    \\
    &\leq 2 \sqrt{\int (\phi_f(x)-p_y(x))^2 \dif\mu_x(x)}.
      \end{align*}

  \end{enumerate}
\end{proof}

The remainder of this technical subsection develops a variety of concentration inequalities used throughout,
most notably the control over the sphere in \Cref{fact:covc:again}.
First, a few standard Gaussian inequalities, included here for completeness.

\begin{lemma} \label{fact:gaussians:1}
  Suppose $W\in\R^{m\times d}$ has iid Gaussian entries $W_{j,k} \sim \cN(0,1)$,
  and let $(w_j^\T)_{j=1}^m$ denote the rows.
  \begin{enumerate}
    \item
      For any $\tau > 0$,
      with probability at least $1-3\delta$,
      \[
        \sum_{j=1}^m \1\sbr{|w_j^\T x| \leq \tau \|x\| }
        \leq m \tau + \sqrt{8m \tau \ln(1/\delta)}.
      \]

    \item
      With probability at least $1-\delta$,
      \[
        \|W\|_2 < \sqrt{m} + \sqrt{d} + \sqrt{2 \ln(1/\delta)}.
      \]

    \item
      With probability at least $1-2\delta$,
      \begin{align*}
        -\|z\| \sqrt{2\ln(1/\delta)}
        &\leq
        \|\srelu(Wz)\| - \bbE \|\srelu(Wz)\|
        \leq
        \|z\| \sqrt{2\ln(1/\delta)},
      \end{align*}
      where
      \[
        \|z\|\del{\sqrt{\frac m 2} - \frac 5 {\sqrt{ 8 m}}}
        \leq
        \bbE \| \srelu(Wz)\|
        \leq
        \|z\|\sqrt{\frac m 2}
        .
      \]

    \item
      With probability at least $1-\delta$, $w\in\R^d$
      with coordinates $w_i \sim \cN(0,1)$ satisfies
      \[
        \|w\| \leq \sqrt{d} + \sqrt{2\ln(1/\delta)}.
      \]

  \end{enumerate}
\end{lemma}

\begin{proof}
  \begin{enumerate}
    \item
      For any row $j$, define an indicator random variable
      \[
        P_j := \1[ |w_j^\T x| \leq \tau\|x\| ].
      \]
By rotational invariance, $P_j = \1[ |w_{j,1}| \leq \tau ]$,
      which by the form of the Gaussian density gives
      \[
        \Pr[P_j = 1] \leq \frac {2\tau}{\sqrt{2\pi}} \leq \tau.
      \]
      As such, by a multiplicative Chernoff bound \citep[Theorem 12.6]{blum_hopcroft_kannan},
      with probability at least $1-3\delta$,
      \[
        \sum_{j=1}^m P_j
        \leq
        m \Pr[P_1 = 1]
        + \sqrt{8 m \Pr[P_1=1] \ln(1/\delta)}
        \leq
        m \tau
        + \sqrt{8 m \tau \ln(1/\delta)},
        \]
        as desired.

    \item
      This is a standard spectral norm concentration bound for Gaussian matrices
      \citep[Theorem II.13]{spectral_concentration},

    \item
      For the expectation, first note for a single row $w^\T$ by rotational invariance
      of the Gaussian that
      \[
        \bbE \srelu(w^\T x)^2 = \|x\|^2 \bbE \srelu(w_1)^2
        = \frac 1 2 \|x\|^2 \bbE w_1^2
        = \frac {\|x\|^2}{2}.
      \]
      As such, for a full matrix $W$, the expected norm can be upper bounded via
      \[
        \bbE \|\srelu(Wx)\|
        \leq
        \sqrt{ \bbE \|\srelu(Wx)\|^2 }
        =
        \sqrt{\frac 1 2 \sum_{i=1}^m \|x\|^2 }
        = \|x\|\sqrt{m/2},
      \]
      and by a second-order lower bound, letting $\tilde x = x/\|x\|$ for convenience,
      and dividing through by $\sqrt{m/2}$ to ease notation,
      \begin{align*}
        \bbE \sqrt{ 2 \|\srelu(Wx)\|^2/m }
        &=
        \|x\|\bbE \sqrt{ 2 \|\srelu(W\tilde x)\|^2/m }
        \\
        &\geq \|x\| \bbE \del{1 + (2\|\srelu(W\tilde x)\|^2/m - 1)/2 - (2\|\srelu(W\tilde x)\|^2/m - 1)^2/2}
        \\
        &= \|x\| \del{1 - \bbE (2\|\srelu(W\tilde x)\|^2/m - 1)^2/2}
        \\
        &= \|x\| \del{\frac 3 2 - \frac {m(m-1)}{2m^2} - \frac {6m}{2m^2} }
        \\
        &= \|x\| \del{ 1 - \frac 5 {2m} }.
      \end{align*}
For the concentration part, note firstly that $\srelu$ is $\ell_2$-Lipschitz
      when applied coordinate-wise, since
      \[
        \| \srelu(u) - \srelu(v)\|^2
        = \sum_{i=1}^m (\srelu(u_i) - \srelu(v_i))^2
        \leq \sum_{i=1}^m (u_i - v_i)^2
        = \|u-v\|^2,
      \]
      and thus
      \[
        \|\srelu(Ax)\| - \|\srelu(Bx)\|
        \leq
        \enVert{\srelu(Ax) - \srelu(Bx) }
        \leq
        \enVert{Ax - Bx }
        \leq
        \enVert{A-B}\|x\|,
      \]
      and thus by standard Gaussian concentration, with probability at least $1-\delta$,
      \[
        \|\srelu(Wx)\| - \bbE \|\srelu(Wx)\|
        \leq
        \|x\| \sqrt{2\ln(1/\delta)},
      \]
      and vice versa.

    \item
      This is a subset of the preceding proof: $w \mapsto \|w\|$ is $1$-Lipschitz,
      thus by standard Gaussian concentration, with probability at least $1-\delta$,
      \[
        \|w\| \leq \bbE \|w\| + \sqrt{2\ln(1/\delta)},
      \]
      where $\bbE\|w\| \leq \sqrt {\bbE \|w\|^2} = \sqrt{d}$.

  \end{enumerate}
\end{proof}

Next, finally, the control over the sphere, \Cref{fact:covc:again}.
This lemma perhaps looks a bit underwhelming or simply abstract or overly complicated,
but is a key tool in many steps of the proofs here; in particular, since it allows consideration
for all $\|x\|\leq 1$, it may be applied over the distribution.
This consideration over the entire sphere contrasts this lemma (and its applications)
from similar inequalities in prior work \citep{allen_deep_opt,gu_polylog}.

\begin{lemma}\label{fact:covc:again}
  Let scalars $R_V \geq 0$, and $\eps \in (0,1/(md))$, and $m \geq \ln(edm)$ be given,
  along with a filter set $\cS_0 \subseteq \R^{m\times d}$,
  and define $\cS := \cS_0 \cap \{ V\in\R^{m\times d} : \|V-W_0\|\leq R_V \}$.
  Let a function $h_V:\R^d\to\R$ be given with parameter $V\in\cS$,
  and define functions
  \[
    \cH := \cbr{x\mapsto h_V(x) + \ip{\nf(x;W_0)}{V-W_0} : V\in\cS}.
  \]
  Moreover, let additional scalars $r_1,r_2,\delta$ satisfy the following conditions.
  \begin{enumerate}
    \item
      For every $x$ and $z$ with $\|x-z\|\leq \eps$,
      then $\sup_{V\in\cS}|h_V(x) - h_V(z)| \leq r_1$.

    \item
      For any fixed $\|x\|\leq 1$, with probability at least $1-\delta$,
then $\sup_{h\in\cH}|h(x)|\leq r_2$.
  \end{enumerate}
  Then with probability at least $1 - (\sqrt{d}/\eps)^d\delta $,
  \[
    \sup_{\|x\|\leq 1}  \sup_{h\in\cH} |h(x)|
    \leq
r_2 + r_1
    +
    11 R_V \rho \del{\frac{\ln(edm/\delta)}{m}}^{1/4}.
\]
\end{lemma}

The proof of \Cref{fact:covc:again} will need two technical lemmas.
The first is a basic property of inner products and arccosine which
also makes a later appearance in \Cref{fact:shallow:barUi:sample}.

\begin{lemma}
  \label{fact:acos}
  If $\|x-z\| \leq \eps$ and $x, z \neq 0$, then
  \[
    1 \geq \ip{\frac {x}{\|x\|}}{\frac {z}{\|z\|}} \geq 1-\frac{2\eps^2}{\|x\|^2},
\qquad\text{and}\qquad
\arccos\del{\ip{\frac {x}{\|x\|}}{\frac {z}{\|z\|}}} \leq \frac{\eps\sqrt{8}}{\|x\|}.
  \]
\end{lemma}
\begin{proof}
The first inequalities follow from
  \begin{align*}
    1
    &\geq \ip{\frac {x}{\|x\|}}{\frac {z}{\|z\|}}
    \\
    &=
    1 - \frac 1 2 \enVert{ \frac {x}{\|x\|} - \frac {z}{\|z\|} }^2
    \\
    &=
    1 - \frac 1 {2\|x\|^2\|z\|^2} \enVert{ x\|z\| - z\|z\|  + z(\|z\|-\|x\|) }^2
    \\
    &\geq
    1 - \frac {\|x-z\|^2\|z\|^2 + \|z\|^2 (\|z\|-\|x\|)^2} {\|x\|^2\|z\|^2}
    \\
    &\geq
    1 - \frac {2\|x-z\|^2\|z\|^2} {\|x\|^2\|z\|^2}
    \\
    &\geq
    1 - \frac {2\eps^2}{\|x\|^2}.
  \end{align*}
  To finish, since $\arccos$ is decreasing along $[0,1]$, and since
  for any $a\in[0,1]$,
  \[
    \arccos(1-a)
    = \int_{1-a}^1 \frac {\dif r}{\sqrt{1-r^2}}
    = \int_{0}^a \frac {\dif r}{\sqrt{2r-r^2}}
    \leq \int_{0}^a \frac {\dif r}{\sqrt{r}}
    = 2\sqrt{a},
  \]
  then
  \[
    \arccos\del{\ip{\frac {x}{\|x\|}}{\frac {z}{\|z\|}}}
    \leq
    \arccos\del{1-\frac {2\eps^2}{\|x\|^2}}
    \leq 2\sqrt{\frac {2\eps^2}{\|x\|^2}}
    = \frac {\eps\sqrt{8}}{\|x\|}.
  \]
\end{proof}

The main heavy lifting in \Cref{fact:covc:again} is encapsulated in the following
concentration inequality.  In words, it controls the behavior of the initial features
within a tiny localized region of the sphere;
the proof of \Cref{fact:covc:again} combines this local control with a discrete cover
of the sphere, together giving control over the entire sphere.

\begin{lemma}\label{fact:covc:helper}
  Let any fixed $\|z\|\leq 1$ be given (independent of $W_0$),
  along with a scalar $\eps>0$ with $\eps \leq 1/(dm)$,
  where $m\geq \ln(edm)$.
  Then, with probability at least $1-\delta$,
  \[
    \sup_{\substack{\|x-z\|\leq \eps\\\|x\|\leq 1}} \|\nf(x;W_0) - \nf(z;W_0)\|^2
    \leq
113 \rho^2 \sqrt{\frac{\ln(edm/\delta)}{m}}.
  \]
\end{lemma}
\begin{proof}
  Throughout the proof, simplify notation via $W := W_0$, and let
  $(w_j^\T)_{j=1}^m$ denote the rows of $W$,
  and furthermore write
  \[
    g(x,z;w) := \frac {\rho^2}{m} \enVert{ x\1[w_j^\T x\geq 0] - z\1[w_j^\T z\geq 0] }^2.
  \]
  Lastly, for any $x\in\R^d$ under consideration, then $\|x\|\leq 1$, so this condition will often
  be implicit.
  Note that
  \begin{align*}
    \sup_{\|x-z\|\leq \eps} \|\nf(x;W) - \nf(z;W)\|^2
    &=
    \frac {\rho^2}{m}
    \sup_{\|x-z\|\leq \eps} \sum_{j=1}^m \|x\1[w_j^\T x\geq 0] - z \1[w_j^\T z\geq 0]\|^2
    \\
    &=
    \sup_{\|x-z\|\leq \eps} \sum_{j=1}^m g(x,z;w_j).
  \end{align*}
  Next note that this quantity, treated as a function of the $m$ rows of $W$,
  satisfies bounded differences with constant $\rho^2/m$:
  letting $W'$ be a copy of $W$ which differs only in a single row $w_i'$,
  and noting $g\geq 0$,
  \begin{align*}
    &
    \envert{
      \sup_{\|x-z\|\leq \eps} \sum_{j=1}^m g(x,z;w_j)
      - 
      \sup_{\|x-z\|\leq \eps} \sum_{j=1}^m g(x,z;w_j')
    }
    \\
    &=
    \envert{
      \sup_{\|x-z\|\leq \eps} \sum_{j=1}^m g(x,z;w_j)
      - 
      \sup_{\|x-z\|\leq \eps}\del[2]{g(x,z;w_i) - g(x,z;w_i) + \sum_{j=1}^m g(x,z;w_j')}
    }
    \\
    &\leq
    \sup_{\|x-z\|\leq \eps}
    \envert{
      g(x,z;w_i') - g(x,z;w_i)
    }
    \leq
    \frac {\rho^2}{m}.
  \end{align*}
  As such, by McDiarmid's inequality, with probability at least $1-\delta$,
  \begin{align}
    \sup_{\|x-z\|\leq \eps} \|\nf(x;W) - \nf(z;W)\|^2
    &\leq
    \sqrt{\rho^4 \ln(1/\delta)/(2m)}
    \notag\\
    &\quad +
    \bbE_W
    \sup_{\|x-z\|\leq \eps} \|\nf(x;W) - \nf(z;W)\|^2.
    \label{eq:blah:1}
  \end{align}

  It remains to analyze this expectation.
  First consider the case that $\|z\|\leq 3\sqrt{\eps}$; then, for any $W$,
  \begin{align}
    \sup_{\|x-z\|\leq \eps} \|\nf(x;W) - \nf(z;W)\|^2
&
    =
    \frac {\rho^2}{m}
    \sup_{\|x-z\|\leq \eps} \sum_{j=1}^m \|x\1[w_j^\T x\geq 0] - z \1[w_j^\T z\geq 0]\|^2
    \notag\\
    &\leq
    \frac {2\rho^2}{m}
    \sup_{\|x-z\|\leq \eps} \sum_{j=1}^m \del{\|x\|^2 + \|z\|^2}
    \notag\\
    &\leq
    \frac {2\rho^2}{m}
    \sum_{j=1}^m \del{16\eps + 9\eps}
    \leq
    50 \eps \rho^2.
    \label{eq:blah:2}
  \end{align}
  For the rest of the proof, suppose $\|z\| > 3\sqrt{\eps}$, which also implies
  $\|x\|> 2\sqrt{\eps}$ for every $x$ satisfying $\|x-z\|\leq \eps$.

  Since $z$ is fixed, and in particular does not
  depend on $W$, we may use the rotational invariance of $W$ to leverage the condition
  $\|x-z\|\leq \eps$.  Specifically, define a matrix $M\in\R^{d\times d}$
  whose first column is $z/\|z\|$, and the remaining columns are orthonormal
  (we can not use $x$ in the definition of $M$, since $x$ varies within the expectation).
  Defining (for any $x$) the two projections
  $x_z := zx^\T z/\|z\|^2$ and $x^\perp := x - x_z$ (whereby $z^\T x^\perp = 0$),
  we may rotate the rows of $W$ by $M$, giving
  \begin{align*}
    \1\sbr{ (Mw_j)^\T z \geq 0 }
    &= \1\sbr{ w_{j,1} \|z\| \geq 0 }
    \\
    &
    = \1\sbr{ w_{j,1}\geq 0 },
    \\
    \1\sbr{ (Mw_j)^\T x \geq 0 }
    &= 
    \1\sbr{ w_j^\T M^\T (x_z + x^\perp) \geq 0 }
    \\
    &=\1\sbr{ w_{j,1} z^\T x /\|z\| \geq - w_j^\T M^\T x^\perp }
    \\
    &=\1\sbr{ w_{j,1}\geq - \frac {w_j^\T M^\T x^\perp}{ z^\T x /\|z\| } },
  \end{align*}
  where the last division does not change the sign due to $\|z-x\|\leq \eps$
  and $\|z\|> 3\sqrt{ \eps}$,
  for instance as verified by upcoming invocations of \Cref{fact:acos}.
  Now let $E_j$ denote the event that for this $w_j$, there exists $\|x-z\|\leq \eps$
  such that these two indicators are not equal.
  Letting $\tau>0$ denote a free parameter to be optimized later, this
  event is implied by the union of two simpler events:
  let $w_{j,2:}\in\R^{d-1}$ denote all but the first coordinate of $w_j$,
  and define
  \[
    E_{j,1} := \sbr{ |w_{j,1}| \leq \tau },
    \qquad
    E_{j,2} :=
\sbr{\sup_{\|x-z\|\leq \eps} \frac {\|w_{j,2:}\|\cdot \|x^\perp\|\cdot \|z\|}{z^\T x} > \tau };
  \]
  by construction (and Cauchy-Schwarz),
  if the negation of both events holds, then the indicators are the same.
  To upper bound the probability of the first event, by the
  form of the Gaussian density,
  \[
    \Pr[E_{j,1}] \leq \tau \sqrt{\frac  2 \pi} < \tau.
  \]
  To control the various terms in $E_{j,2}$, firstly by \Cref{fact:gaussians:1},
  with probability at least $1 - \eps$, then
  \[
    \|w_{j,2:}\| \leq \sqrt{d-1} + \sqrt{2\ln(1/\eps)} \leq \sqrt{2d - 2 + 4\ln(1/\eps)};
  \]
  this will be the only step of the derivation controlling $\Pr[E_{j,2}]$, and note that it
  depends only on $w_j$ and $z$ and not on any specific $x$.
  Next, by \Cref{fact:acos}, for any $\|x-z\|\leq \eps$, since $\|x\|\geq 2\eps$
  (whereby $2\eps^2/\|x\|^2<1$),
  \begin{align*}
    \|x^\perp\|^2
    &= \|x\|^2 - \frac {(z^\T x)^2}{\|z\|^2}
    = \|x\|^2\del{1 - \sbr{\frac {z^\T x}{\|x\|\|z\|}}^2}
    \\
    &\leq \|x\|^2\del{1 - \sbr{1 - \frac {2\eps^2}{\|x\|^2}}^2 }
    =
    4\eps^2 - \frac {4\eps^4}{\|x\|^2}
    \leq
    4\eps^2.
  \end{align*}
  Similarly by \Cref{fact:acos}, using $\eps \leq 1$,
  \begin{align*}
    \frac {z^\T x}{\|z\|} \geq \|x\| - \frac {2\eps^2}{\|x\|}
    > 2\sqrt{\eps} - \eps^{1.5}
    \geq \sqrt{\eps}.
  \end{align*}
  Combining all these pieces, with probability at least $1-\eps$,
  \[
    \frac{\|w_{j,2:}\|\cdot \|x^\perp\|\cdot \|z\|}{z^\T x}
    \leq
    \sqrt{2d - 2 + 4\ln(1/\eps)} \del{\frac{2\eps}{ \sqrt{\eps}}}
\leq 4\sqrt{d \eps \ln(e/\eps)}.
  \]
  This right hand side does not depend on the specific choice of $x$,
  and holds for any $\|x-z\|\leq \eps$.
  As such, set $\tau := 4 \sqrt{d\eps \ln(e/\eps)}$, whereby
  \[
    \Pr[E_j] \leq \Pr[E_{j,1}] + \Pr[E_{j,2}] \leq \tau + \eps.
  \]
  Moreover, by a multiplicative Chernoff bound \citep[Theorem 12.6]{blum_hopcroft_kannan},
  with probability at least $1-3\eps$,
  the events $(E_j)_{j=1}^m$
  hold for at most $m_\tau:= m(\tau + \eps) + \sqrt{8 m (\tau +\eps)\ln(1/\eps)}$ rows.
  Now let $E_\tau$ denote the event that $(E_j)_{j=1}^m$ holds for at most $m_\tau$ rows.
  Then
  \begin{align}
    &
    \bbE_W
    \sup_{\|x-z\|\leq \eps} \|\nf(x;W) - \nf(z;W)\|^2.
    \notag\\
    &=
    \bbE_W
    \sbr{
      \sup_{\|x-z\|\leq \eps} \|\nf(x;W) - \nf(z;W)\|^2
    \ | \ E_\tau } \Pr[E_\tau]
    \notag\\
    &\quad + 
    \bbE_W
    \sbr{
      \sup_{\|x-z\|\leq \eps} \|\nf(x;W) - \nf(z;W)\|^2
    \ | \ E_\tau^c } \Pr[E_\tau^c]
    \notag\\
    &\leq
    2\rho^2
    \sup_{\|x-z\|\leq \eps}
    \del{
      \frac m m \|x-z\|^2
      +
      \frac {m_\tau}{m}(\|x\|^2 + \|z\|^2)
    }
    +
    \sup_{\|x-z\|\leq \eps}
    \del{
      3\eps (\|x\|^2 +\|z\|^2)
    }
    \notag\\
    &\leq
    2\rho^2
    \del{
      \eps^2
      +
      \frac {2m_\tau}{m}
      +
      6\eps
    }.
    \label{eq:blah:3}
  \end{align}

  The proof will now be completed by returning to the McDiarmid application resulting in \cref{eq:blah:1},
  and combining all preceding bounds.
  Starting with a simplification via the assumption $\eps \leq 1 / (dm)$
  and $m\geq \ln(edm)$, note
  \begin{align*}
    \tau
    &=
    4 \sqrt{d\eps\ln(e/\eps)}
    \leq 4 \sqrt{\frac {\ln(edm)} m},
    \\
    \frac {m_\tau}{m}
    &=
    \tau + \eps + \sqrt{8 (\tau +\eps)\ln(1/\eps)/m}
    \\
    &
    \leq
    5 \sqrt{\frac {\ln(edm)} m}
    +
    \sqrt{\frac {40 \sqrt{\ln(edm)}\ln(edm)}{m^{3/2}}}
    \leq
    12 \sqrt{\frac {\ln(edm)} m}
    .
  \end{align*}
  Combining the preceding simplifications with \cref{eq:blah:2,eq:blah:3},
  continuing from the McDiarmid application in \cref{eq:blah:1},
  with probability at least $1-\delta$,
  \begin{align*}
    \sup_{\substack{\|x-z\|\leq \eps\\\|x\|\leq 1}} \|\nf(x;W_0) - \nf(z;W_0)\|^2
    &\leq
    \rho^2\del{
      \sqrt{\frac{\ln(1/\delta)}{2m}}
      + 50\eps
      +
      2
      \del{
        \eps^2
        +
        \frac {2m_\tau}{m}
        +
        6\eps
      }
    }
    \\
    &\leq
    \rho^2\del{
      \sqrt{\frac{\ln(1/\delta)}{2m}}
      + \frac{50}{md}
      +
      \frac {2}{m^2d^2}
      +
      48
      \sqrt{\frac {\ln(edm)} m}
      +
      \frac {12}{md}
    }
    \\
    &
    \leq
    113 \rho^2 \sqrt{\frac{\ln(edm/\delta)}{m}}.
  \end{align*}
\end{proof}

Finally, the proof of \Cref{fact:covc:again} via the preceding technical lemmas.

\begin{proof}[Proof of \Cref{fact:covc:again}]
  Let $\cC$ denote a cover of each coordinate of $\|x\|\leq 1$ at scale
  $\eps/\sqrt{d}$, meaning $|\cC| \leq (\sqrt{d}/\eps)^d$
  (the grid elements can be $2\eps/\sqrt{d}$ apart),
  and for any $\|x\|\leq 1$, there exists $z\in\cC$ with
  \[
    \|z-x\| =\sqrt{\sum_{i=1}^d (z_i -x_i)^2} \leq \eps.
  \]
  This cover $\cC$ will be used throughout the proof; it is crucial
  that its construction makes no reference to $W_0$, and in particular
  that the cover elements are independent of $W_0$.

  Union bound together and discard $|\cC|\delta$ failure probability so that
  for every $z\in \cC$, then $\sup_{h\in\cH} |h(z)| \leq r_2$.
  Additionally union bound together and discard $|\cC|\delta$ failure probability
  corresponding to instantiating \Cref{fact:covc:helper} for each $z\in\cC$,
  whereby
  \[
    \max_{z\in\cC}
    \sup_{\substack{\|x-z\|\leq \eps\\\|x\|\leq 1}} \|\nf(x;W_0) - \nf(z;W_0)\|^2
    \leq
113 \rho^2 \sqrt{\frac{\ln(edm/\delta)}{m}}.
  \]

  Now let an arbitrary $\|x\|\leq 1$ be given, and let $z\in\cC$ be a nearest cover
  element, whereby $\|z-x\|\leq \eps$. Then
  \begin{align*}
    \sup_{h\in\cH} |h(x)|
    &\leq
    \sup_{h\in\cH} \del[2]{ |h(z)| + |h(z) - h(x)| }
    \\
    &\leq
    r_2
    +
    \sup_{V\in\cS}
    |h_V(z) - h_V(x)|
+
    \sup_{V\in\cS}
    |\ip{\nf(x;W_0) - \nf(z;W_0)}{V-W_0}|
    \\
    &\leq
    r_2
    +
    r_1
    +
    \sup_{V\in\cS}
    \enVert{ \nf(x;W_0) - \nf(z;W_0)} \cdot \enVert{V-W_0}
\\
    &\leq
    r_2 + r_1
    +
    11 R_V \rho \del{\frac{\ln(edm/\delta)}{m}}^{1/4}.
\end{align*}
\end{proof}

As a first application of \Cref{fact:covc:again},
the range of the mappings can be bounded for all $\|x\|\leq 1$, which is
used later in the generalization analysis.

\begin{lemma}\label{fact:gaussians:2}
  Let $R_V>0$ be given.
  \begin{enumerate}

    \item
      \label{fact:gaussians:initial_mapping_boundedness}
      For any $x\in\R^d$,
      with probability at least $1-3\delta$,
      every $V\in\R^{m\times d}$ satisfies
      \[
        \envert{\ip{\nf(x;W_0)}{V}}
        \leq
        \rho\|x\|\del{\|V-W_0\|_\tF + 2\ln(1/\delta)}.
      \]

    \item
      Suppose $R_V \geq 1$ and $m\geq \ln(emd)$.
      With probability at least $1-(1 + 3(md^{3/2})^d)\delta$,
      \[
        \sup_{\|V-W_0\|\leq R_V}\sup_{\|x\|\leq 1}
        \envert{\ip{\nf(x;W_0)}{V}}
        \leq
18 R_V \rho \ln(emd/\delta).
      \]

  \end{enumerate}
\end{lemma}

\begin{proof}
  For convenience throughout the proof, write $W := W_0$.
  \begin{enumerate}

    \item
      Splitting terms via $V = V- W + W$,
      \begin{align*}
        \envert{\ip{\nf(x;W)}{V}}
        &\leq
        \envert{\ip{\nf(x;W)}{W}}
        +\envert{\ip{\nf(x;W)}{V-W}}.
      \end{align*}
      For the first term, since $W$ is independent of $a$ and can be treated as fixed,
      by Hoeffding's inequality, with probability at least $1-2\delta$ over the draw of $a$,
      \begin{align*}
        \envert{\ip{\nf(x;W)}{W}}
        =
        \envert{f(x;W)}
\leq
        \frac {\rho}{\sqrt m} \|\srelu(Wx)\| \sqrt{\ln(1/\delta)/2}.
      \end{align*}
      By \Cref{fact:gaussians:1}, with additional failure probability $\delta$,
      \[
        \|\srelu(Wx)\|
        \leq
        \bbE \|\srelu(Wx)\| + \|x\| \sqrt{2\ln(1/\delta)}
        \leq
        \|x\|\del{\sqrt{m/2} + \sqrt{2\ln(1/\delta)}}.
      \]
      Together,
      \[
        \envert{\ip{\nf(x;W)}{W}}
        \leq
        \rho \|x\|\del{1 + \sqrt{2\ln(1/\delta)/m}}\sqrt{\ln(1/\delta)/2}.
      \]
      For the second term, due to the scale of the first term, it suffices to worst-case
      everything: by Cauchy-Schwarz,
      \[
        \envert{\ip{\nf(x;W)}{V-W}}
        \leq
        \|\nf(x;W)\|_\tF\cdot\|V-W\|_\tF
        \leq
        \rho \|x\|\cdot \|V-W\|_\tF.
      \]
      Combining everything, with probability at least $1-3\delta$,
      \begin{align*}
        \envert{\ip{\nf(x;W)}{V}}
        &\leq
        \rho\|x\|\del{\|V-W\|_\tF + \sqrt{\ln(1/\delta)/2} + \ln(1/\delta)/\sqrt{m}}
        \\
        &\leq
        \rho\|x\|\del{\|V-W\|_\tF + 2\ln(1/\delta)}
      \end{align*}

    \item
This item proceeds by combining the previous item with the covering argument
      from \Cref{fact:covc:again}.  Concretely, define the function
      \[
        h_V(x) := \ff{0}(x);
      \]
      that is, $h_V$ has no dependence on $V\in\R^{m\times d}$, but note that
      \[
        \ip{\nf(x;W)}{V}
        =
        \ip{\nf(x;W)}{V-W}
        +
        \ip{\nf(x;W)}{W}
        =
        \ip{\nf(x;W)}{V-W}
        +
        h_V(x),
      \]
      which is precisely the expression controlled by \Cref{fact:covc:again}.
      Let $\cH$ denote the class of functions defined there.

      By the preceding item, for any fixed $\|x\|\leq 1$, with probability at least $1-\delta$,
      \[
        \sup_{h\in\cH} |h(x)| \leq
        \rho\del{R_V + 2\ln(1/\delta)}
        =: r_2.
      \]
      Moreover, by \Cref{fact:gaussians:1}, with probability at least $1-\delta$,
      then $\|W\|_2 \leq \sqrt{m} + \sqrt{d} + \sqrt{2\ln(1/\delta)}$,
      thus for any $\|x-z\|\leq \eps$, with $\eps$ to be determined later,
      \begin{align*}
        |f(x;W) - f(z;W)|
        &\leq
        \frac {\rho}{\sqrt m} \|a\|\cdot\|W(x-z)\|
        \leq
        \rho \|W\|_2 \|x-z\|
        \\
        &\leq \rho (\sqrt m + \sqrt d + \sqrt{2\ln(1/\delta)} ) \eps
        =: r_1.
      \end{align*}
      As such, by \Cref{fact:covc:again},
      choosing $\eps := 1/(md)$ and $\cS_0 = \R^{m\times d}$,
      with probability at least $1-3(md^{3/2})^d\delta$,
      \begin{align*}
        \sup_{\|V-W\|\leq R_V} \sup_{\|x\|\leq 1} h_V(x)
        &\leq
        r_2 + r_1 + 
        11 R_V \rho \del{\frac{\ln(edm/\delta)}{m}}^{1/4}.
\\
        &\leq 
        \rho\del{R_V + 2\ln(1/\delta)}
        \\
        &\quad
         + \rho (\sqrt m + \sqrt d + \sqrt{2\ln(1/\delta)} ) \eps
         \\
        &\quad
        + 
11 R_V \rho \del{\frac{\ln(edm/\delta)}{m}}^{1/4}.
        \\
        &\leq
18 R_V \rho \ln(emd/\delta).
      \end{align*}

  \end{enumerate}
\end{proof}

Next, the linear approximation bounds; the last two items use \Cref{fact:covc:again}
to control all points on the sphere.
As mentioned before, this is in contrast to
prior presentations of linear approximation inequalities,
which only establish the bounds on the finite training sample \citep{gu_polylog,allen_deep_opt}.
Note that the bounds over the sphere have a more restrictive statement; the present
proof does not handle the more general form presented for a finite sample.

\begin{lemma}[See also \Cref{fact:shallow:linearization:frob:simplified}]\label{fact:shallow:linearization:frob}
  Let scalars $\delta>0$ and $R_V\geq 1$ and $R_B\geq 0$ be given.
  \begin{enumerate}
    \item
      For any fixed $x\in\R^d$,
      with probability at least $1-3\delta$,
      for any $V\in\R^{m\times d}$ and $B\in\R^{m\times d}$ with $\|V-W_0\|\leq R_V$
      and $\|B-W_0\|\leq R_B$,
      \[
        \envert{\ip{\nf(x;V) - \nf(x;W_0)}{B}}
        \leq
        \frac {3\rho \|x\|\del{R_B + 2R_V}R_V^{1/3}\ln(e/\delta)^{1/4}}{m^{1/6}}
        =: \tau_1.
      \]

    \item
      Let $\tau_1$ be as in the previous part.
      With probability at least $1-3n\delta$,
      \[
        \sup_{\|W_i - W_0\| \leq R_V}
        \sup_{\|W_j - W_0\| \leq R_V}
        \sup_{\|B - W_0\|\leq R_B}
        \frac {\hcR^{(i)}(B)}{\hcR^{(j)}(B)} \leq e^{2\tau_1}.
      \]

    \item
Suppose $m\geq \ln(edm)$.
      With probability at least $1- (1 + 3 (d^2m)^{d})\delta$, 
      \[
        \sup_{\|V-W_0\|\leq R_V}
        \sup_{\|x\|\leq 1}
        \envert{\ip{\nf(x;V) - \nf(x;W_0)}{V}}
        \leq
\frac {25 \rho R_V^{4/3} \sqrt{\ln(edm/\delta)}}{m^{1/6}} =: \tau_3.
      \]

    \item
      Let $\tau_3$ be as in the previous part
      and again suppose $m \geq \ln(edm)$.
      With probability at least $1- (1 + 3 (d^2m)^{d})\delta$, 
      \[
        \sup_{\|W_i - W_0\| \leq R_V}
        \frac {\cR(W_i)}{\cR^{(0)}(W_i)} \leq e^{\tau_3}.
      \]
  \end{enumerate}
\end{lemma}
\begin{proof}[Proof of \Cref{fact:shallow:linearization:frob,fact:shallow:linearization:frob:simplified}]
  The first item implies the second via \Cref{fact:logistic:error}, and moreover implies
  the third item via \Cref{fact:covc:again}.
  Similarly, the third item implies the fourth via \Cref{fact:logistic:error}.
  Throughout the proof, write $W := W_0$ with rows $(w_j^\T)_{j=1}^m$ for convenience.

  \begin{enumerate}
    \item
      Fix $x\in\R^d$.
      Fix a parameter $r>0$, which will be optimized at the end of the proof.
      Let $V$ and $B$ be given with $\|V-W\|\leq R_V$ and $\|B-W\|\leq R_B$.

      Define the sets
      \begin{align*}
        S_1
        &:=
        \cbr{ j \in [m] : |w_{j}^\T x| \leq r \|x\| },
        \\
        S_2
        &:=
        \cbr{ j \in [m] : \|v_j-w_j\| \geq r }
        \\
        S &:= S_1 \cup S_2.
      \end{align*}
      By \Cref{fact:gaussians:1}, with probability at least $1-3\delta$,
      \[
        |S_1|
        \leq r m + \sqrt{8 r m \ln(1/\delta)}.
      \]
      On the other hand,
      \[
        R_V^2 \geq \|V-W\|^2 \geq \sum_{j\in S_2} \|v_j-w_j\|^2 \geq |S_2| r^2,
      \]
      meaning $|S_2| \leq R_V^2 / r^2$.
      For any $j\not \in S$, if $w_j^\T x > 0$,
      then
      \[
        v_j^\T x
        \geq w_j^\T x - \|v_j-w_j\|\cdot\|x\|
        > \|x\| \del{ r - r }
        = 0,
      \]
      meaning $\1[w_j^\T x \geq 0] = \1[v_j^\T x \geq 0]$; the case that
      $j\not\in S$ and $w_j^\T x < 0$ is analogous.  Together,
      \[
        |S| \leq r m + \sqrt{8 r m \ln(1/\delta)} + \frac {R_V^2}{r^2}
        \quad
        \textup{and}
        \quad
        j\not\in S \Longrightarrow \1[w_j^\T x \geq 0 ] = \1[v_j^\T x\geq 0].
      \]

      Continuing,
      \begin{align*}
        &\frac {\sqrt{m}}{\rho} \envert{\ip{\nf(x;V) - \nf(x;W)}{B}}
        \\
        &\leq
        \frac {\sqrt{m}}{\rho} 
        \envert{\ip{\nf(x;V) - \nf(x;W)}{V}}
        +
        \frac {\sqrt{m}}{\rho} 
        \envert{\ip{\nf(x;V) - \nf(x;W)}{V-B}}
        \\
        &=
        \envert{ a^\T \del{\diag(\1[V^\T x\geq 0]) -\diag(\1[W^\T x\geq0])} Vx }
        \\
        &\quad +
        \envert{ a^\T \del{\diag(\1[V^\T x\geq 0]) -\diag(\1[W^\T x\geq0])} (V-B)x }.
      \end{align*}
      Handling these two terms separately, the second term is easier:
      by Cauchy-Schwarz,
      \begin{align*}
        \envert{ a^\T \del{\diag(\1[V^\T x\geq 0]) -\diag(\1[W^\T x\geq0])} (V-B)x }
        &\leq \sqrt{|S|} \enVert{(V- W - (B-W))x}
        \\
        &\leq \sqrt{|S|} \del{ R_V + R_B } \|x\|.
      \end{align*}
      For the first term,
      \begin{align*}
        \envert{ a^\T \del{\diag(\1[V^\T x\geq 0]) -\diag(\1[W^\T x\geq0])} Vx }
        \leq \sum_{j=1}^m \1[\sgn(v_j^\T x) \neq \sgn(w_j^\T x)] \cdot | v_j^\T x |.
      \end{align*}
      If $v_j^\T x$ and $w_j^\T x$ have different signs, then
      $|v_j^\T x| \leq |v_j^\T x - w_j^\T x| \leq \|v_j-w_j\| \cdot \|x\|$; plugging this in,
      by Cauchy-Schwarz,
      \begin{align*}
        \sum_{j=1}^m \1[\sgn(v_j^\T x) \neq \sgn(w_j^\T x)] \cdot | v_j^\T x |
        &\leq
        \sum_{j=1}^m \1[\sgn(v_j^\T x) \neq \sgn(w_j^\T x)] \cdot \|v_j - w_j\| \cdot \|x\|
        \\
        &\leq
        \sum_{j\in S}  \|v_j - w_j\| \cdot \|x\|
        \\
        &\leq
        \sqrt{|S|}\|V-W\|_\tF \|x\|
        \\
        &\leq
        R_V \sqrt{|S|} \|x\|.
      \end{align*}
      Combining these derivations,
      \begin{align*}
        \envert{\ip{\nf(x;V) - \nf(x;W)}{B}}
        &\leq
        \frac {\rho}{\sqrt m} \del{ 
          \sqrt{|S|} \del{ R_V + R_B } \|x\|
          +
        R_V \sqrt{|S|} \|x\| }
        \\
        &\leq
        \frac {\rho \sqrt{|S|}\|x\| \del{2R_V + R_B}}{\sqrt m}.
      \end{align*}
      Rearranging, and expanding the definition of $|S|$
      with the choice $r := R_V^{2/3} m^{-1/3}$, and using $R_V\geq 1$,
      \begin{align*}
\envert{\ip{\nf(x;V) - \nf(x;W)}{B}}
&\leq
        \frac {\rho \|x\|\del{R_B + 2R_V}}{\sqrt m} \sqrt{
          r m + \sqrt{8 r m \ln(1/\delta)} + \frac {R_V^2}{r^2}}
          \\
        &\leq
        \frac {\rho \|x\|\del{R_B + 2R_V}R_V^{1/3}m^{1/3}\ln(e/\delta)^{1/4}}{\sqrt m}
        \sqrt{1 + \sqrt 8 + 1}
          \\
        &\leq
        \frac {3\rho \|x\|\del{R_B + 2R_V}R_V^{1/3}\ln(e/\delta)^{1/4}}{m^{1/6}}.
      \end{align*}

    \item
      Union bounding the previous part over all $(x_k)_{k=1}^n$,
      with probability at least $1-\delta$,
      for any iterations $(i,j)$ 
      and for any matrices $(W_i,W_j,B)$ satisfying
      $\max\{\|W_i - W_0\|, \|W_j-W_0\|, \|B-W_0\|\}\leq R_V$
      \[
        \max_{k}
        \envert{\ip{\nf(x_k;W_i) - \nf(x_k;W)}{B}}
        \leq \tau_1.
      \]
      In particular, by \Cref{fact:logistic:error},
      \[
        e^{-\tau_1}
        \leq
        \frac {\hcRi(B)}{\hcR^{(0)}(B)}
        \leq
        e^{\tau_1}.
      \]
      Applying this twice gives
      \[
        e^{-2\tau_1}
        \leq
        \frac {\hcRi(B)}{\hcR^{(0)}(B)}
        \del{\frac {\hcR^{(0)}(B)}{\hcR^{(j)}(B)}}
        =
        \frac {\hcRi(B)}{\hcR^{(j)}(B)}
        \leq e^{2\tau_1}.
      \]

    \item
      This part follows from the first via \Cref{fact:covc:again}.
      As such, for every $\|V-W\|\leq R_V$, define
      \[
        h_V(x) := f(x;W) - f(x;V);
      \]
      by this choice,
      \begin{align*}
        \ip{\nf(x;V) - \nf(x;W)}{V}
        &=
        \ip{\nf(x;V)}{V} - \ip{\nf(x;W)}{W} - \ip{\nf(x;W)}{V-W}
        \\
        &=
        f(x;V) - f(x;W) - \ip{\nf(x;W)}{V-W}
        \\
        &=
        -h_V(x) - \ip{\nf(x;W)}{V-W},
      \end{align*}
      which matches the (negation of) functions considered in the function class $\cH$
      in \Cref{fact:covc:again}.

      By the previous part, with $R_B := 0$,
      for any fixed $\|x\|\leq 1$, with probability at least $1-3\delta$,
      \[
        \sup_{h\in\cH} |h(x)| \leq 
        \frac {6\rho R_V^{4/3}\ln(e/\delta)^{1/4}}{m^{1/6}} =: r_2.
      \]
      Next, with probability at least $1-\delta$, \Cref{fact:gaussians:1} gives
      \[
        \|W\|_2 \leq \sqrt{m} + \sqrt{d} + \sqrt{2\ln(1/\delta)},
      \]
      and thus for any $\|x-z\|\leq \eps$,
      since the ReLU is $1$-Lipschitz even when applied
      to vectors,
      \begin{align*}
        |h_V(x) - h_V(z)|
        &\leq
        |f(x;V) - f(z;V)| + |f(x;W) - f(z;W)|
        \\
        &\leq
        \rho \|(V-W+W)(x-z)\| + \rho \|W(x-z)\|
        \\
        &\leq
        2\rho \eps (R_V/2 + \sqrt{m} + \sqrt{d} + \sqrt{2\ln(1/\delta)} )
        =: r_1.
      \end{align*}
      Together, by \Cref{fact:covc:again},
      choosing $\eps := 1/ (dm)$ and $\cS_0 := \R^{m\times d}$,
      with probability at least $1-(1+3(md^{3/2})^{d})\delta$,
      \begin{align*}
        \sup_{\|x\|\leq 1}
        \sup_{h\in\cH}
        |h(x)|
        &\leq 
        r_2 + r_1
+ 11 R_V \rho \del{\frac{\ln(edm/\delta)}{m}}^{1/4}
\leq
        \frac {25 \rho R_V^{4/3} \sqrt{\ln(edm/\delta)}}{m^{1/6}}.
      \end{align*}

    \item
      By the previous item,
      with probability at least $1-(1+3(md^{3/2})^{d})\delta$,
      \[
        \sup_{\|W_i - W_0\|\leq R_V}
        \sup_{\|x\|\leq 1}
        \envert{ \ff{0}(x;W_i) - f(x;W_i) } \leq \tau_3.
      \]
      Consequently, by \Cref{fact:logistic:error},
      for any $W_i$ with $\|W_i - W_0\| \leq R_V$,
      \[
        \cR(W_i)
        = \bbE_{x,y} \ell(y f(x;W_i))
        \leq
        e^{\tau_3} 
        \bbE_{x,y} \ell(y \ff{0}(x;W_i))
        =
        e^{\tau_3} 
        \cR^{(0)}(W_i).
      \]

\end{enumerate}
\end{proof}

\subsection{Generalization proofs}

As mentioned before, the usual hard part of such a proof is the Rademacher complexity estimate,
but here it is easy: linear predictors, as this bound is applied after linear approximation.
The difficult step is to control the range, which was presented before
in \Cref{fact:gaussians:2}, which invokes the sphere control technique in
\Cref{fact:covc:again}.

\begin{lemma}\label{fact:shallow:gen:frob:1}
  Let $R_V\geq 1$ and $m\geq \ln(edm)$ be given.
  With probability at least $1-6\delta$,
  \begin{align*}
    \sup_{\|V-W_0\|\leq R_V}
    \cR^{(0)}(V)
    -
    \hcR^{(0)}(V)
    \leq
\frac {80\rho R_V\del{d\ln(em^2d^3/\delta)}^{3/2}}{\sqrt{n}}.
  \end{align*}
  Similarly, the negation of this bound holds with probability at least $1-6\delta$.
\end{lemma}

\begin{proof}This proof will use a constant $\delta_0$, chosen at the end.
  First note that the Rademacher complexity is as for linear predictors:
  \begin{align*}
    n\Rad\del{\cbr{x \mapsto \ip{\nf(x;W_0)}{V} : \|V-W_0\|\leq R_V}}
    &= \bbE_\eps \sup_{V\in\cV} \sum_{k=1}^n\eps_k \ip{\nf(x_k;W_0)}{V}
    \\
    &= \bbE_\eps \sup_{V\in\cV} \sum_{k=1}^n\eps_k \ip{\nf(x_k;W_0)}{V- W_0 + W_0}
    \\
    &= \bbE_\eps \sup_{V\in\cV} \sum_{k=1}^n\eps_k \ip{\nf(x_k;W_0)}{V- W_0}
    \\
    &\leq \|V-W_0\|_\tF \sqrt{\sum_{k=1}^n \|\nf(x_k;W_0)\|^2}
    \\
    &\leq \rho R_V \sqrt{n}.
  \end{align*}
  Next, by \Cref{fact:gaussians:2},
  with probability at least $1-(1+3(md^2)^d)\delta_0$,
  the mappings $(x,y) \mapsto \ell(y\ff{0}(x;V))$
  are nonnegative, centered at $\ell(0)$,
  and vary by at most
  $18 \rho R_V\ln(emd/\delta_0)$,
  thus take their amplitude to be $36\rho R_V\ln(emd/\delta_0)$ for simplicity.
As such, since $\ell$ is $1$-Lipschitz,
  by a standard Rademacher bound \citep{shai_shai_book},
  with additional failure probability at most $2\delta_0$,
  \begin{align*}
    \sup_{\|V-W_0\|\leq R_V} \cR^{(0)}(V) - \hcR^{(0)}(V)
    &\leq
    \frac {2 \rho R_V}{\sqrt n}
    +\frac {108 \rho R_V \ln(emd/\delta_0)\sqrt{\ln(1/\delta_0)}}{\sqrt{2n}}
    \\
    &\leq
    \frac {80\rho R_V\ln(emd/\delta_0)^{3/2}}{\sqrt{n}},
  \end{align*}
  and the bound is complete by noting the total failure probability was at most
  $(3+3(md^2)^d)\delta_0\leq 6(md^2)^d\delta_0$,
  and setting $\delta_0 := \delta/(md^2)^d$ and simplifying.

  For the reverse inequality, it follows by negating every element in the loss class
  and repeating the proof.
\end{proof}

\subsection{Optimization proofs}

First, a smoothness inequality which fixes the feature mapping across a pair of iterates.
This \namecref{fact:shallow:smooth}
doesn't seem to have appeared before, but is not necessarily an improvement, other
than allowing slightly larger step sizes.

\begin{lemma}\label{fact:shallow:smooth}
  For any step size $\eta\geq 0$,
  \[
    \eta(1-\eta\rho^2/8)\|\nabla\hcR(W_i)\|^2
    \leq \hcRi(W_i) - \hcRi(W_{i+1}).
  \]
  If $\eta \leq 8/\rho^2$, then $\hcRi(W_{i+1}) \leq \hcRi(W_i)$,
  and any choice $\eta \leq 4/\rho^2$ grants
  \[
    \frac \eta 2 \|\nabla\hcR(W_i)\|^2
    \leq \hcRi(W_i) - \hcRi(W_{i+1}).
  \]
\end{lemma}
\begin{proof}
  For notational convenience, define
  $g_k(W) := y_k f(x_k;W)$ and $\ggi_k(W) := y_k \ffi(x_k;W)$,
  whereby $\nabla g_k(W) = y_k \nabla f(x_k;W)$.
  Since $\ell$ is $\nicefrac 1 4$-smooth and since, for every example $(x_k,y_k)$,
  $\|\nf(x_k;V)\|^2 = \rho^2\sum_{j=1}^m \|a_j \1[w_j^\T x_k\geq 0] x_k\|^2/m \leq 1$,
  then
  \begin{align*}
    \ell(\ggi_k(W_{i+1}))
    &\leq
    \ell(\ggi_k(W_{i}))
    +
    \ell'(\ggi_k(W_{i}))(\ggi_k(W_{i+1}) - \ggi_k(W_i))
+ \frac 1 8 \del{\ggi_k(W_{i+1}) - \ggi_k(W_i)}^2
    \\
    &= 
    \ell(\ggi_k(W_{i}))
    +
    \ip{\ell'(\ggi_k(W_{i}))\nabla g_k(W_i)}{W_{i+1} - W_i}
    + \frac 1 8 \ip{\nabla g_k(W_i)}{W_{i+1} - W_i}^2
    \\
    &= 
    \ell(\ggi_k(W_{i}))
    -\eta \ip{\ell'(\ggi_k(W_{i}))  \nabla g_k (W_i)}{\nabla \hcR(W_i)}
    + \frac 1 8 \ip{\nabla g_k(W_i)}{\eta \nhcR(W_i)}^2
    \\
    &\leq
    \ell(\ggi_k(W_{i}))
    -\eta \ip{\ell'(\ggi_k(W_{i})) \nabla g_k (W_i)}{\nabla \hcR(W_i)}
    + \frac {\rho^2 \eta^2} 8 \enVert{\nabla\hcR(W_i)}^2,
  \end{align*}
  which after averaging over examples gives
  \begin{align*}
    \hcRi(W_{i+1})
    &\leq
    \hcRi(W_i)
    - \frac {\eta}{n} \sum_{k=1}^n
    \ip{\ell'(\ggi_k(W_{i})) \nabla g_k(W_i)}{\nabla \hcR(W_i)}
    + \frac {\rho^2\eta^2} 8 \enVert{\nabla\hcR(W_i)}^2
    \\
    &=
    \hcRi(W_i)
    - \eta(1 - \rho^2 \eta/8) \enVert{\nabla\hcR(W_i)}^2,
  \end{align*}
  which rearranges to give the first inequality.
  Lastly, note if $\eta \leq 4/\rho^2$, then
  $\eta \del{1-\nicefrac{\rho^2\eta}{8}} \geq \nicefrac {\eta}{2}$.
\end{proof}

Next, the familiar regret inequality, making use of feature mappings induced by specific
gradient descent iterates.  Note that this inequality does not need to make any assumptions
on nonlinearity and activation changes, though such effects must be controlled in the eventual
application of this bound.

\begin{lemma}\label{fact:shallow:magic}
  For any step size $\eta \leq 4/\rho^2$, any $Z\in \R^{m\times d}$ and any $t$,
  \[
    \|W_t - Z\|^2 + 2\eta \sum_{i<t} \hcRi(W_{i+1})
    \leq
    \|W_0 - Z\|^2 + 2\eta \sum_{i<t} \hcRi(Z).
  \]
\end{lemma}
\begin{proof}
  As usual, using \Cref{fact:shallow:smooth},
  \begin{align*}
    \|W_{i+1} - Z\|^2
    &= \|W_i - Z\|^2 - 2 \eta \ip{\nabla\hcR(W_i)}{W_i - Z} + \eta^2 \|\nabla\hcR(W_i)\|^2
    \\
    &\leq \|W_i - Z\|^2 + 2 \eta \ip{\nabla\hcR(W_i)}{Z-W_i}
    + 2\eta \del{\hcRi(W_i) - \hcRi(W_{i+1})},
  \end{align*}
  where
  \begin{align*}
    \ip{\nabla\hcR(W_i)}{Z - W_i}
    &=
    \frac 1 n \sum_k
    \ell'(y_kf(x_k;W_i))\ip{y_k \nabla f(x_k;W_i)}{Z - W_i}
    \\
    &=
    \frac 1 n \sum_k
    \ell'(y_kf(x_k;W_i))\del{y_k \ffi(x_k;Z) - y_k \ffi(x_k;W_i)}
    \\
    &\leq
    \frac 1 n \sum_k
    \del{
    \ell(y_k\ffi(x_k;Z))
  - \ell(y_k\ffi(x_k;W_i))}
    \\
    &= \hcRi(Z) - \hcRi(W_i),
  \end{align*}
  together giving
  \begin{align*}
    \|W_{i+1} - Z\|^2
&\leq \|W_i - Z\|^2
    + 2 \eta \del{ \hcRi(Z) - \hcRi(W_{i+1}) },
  \end{align*}
  which after telescoping and rearranging gives the final bound.
\end{proof}

Lastly, the proof of \Cref{fact:shallow:magic:linearized}, 
the central optimization guarantee,
which immediately yields the bulk of \Cref{fact:main}.

\begin{proof}[Proof of \Cref{fact:shallow:magic:linearized}]
The start of this proof establishes a few inequalities used throughout.
  By the second part of \Cref{fact:shallow:linearization:frob},
  with probability at least $1-3n\delta$,
  for any iterations $(i,j)$ with $\|W_i - W_0\|\leq B$ and $\|W_j - W_0\|\leq B$,
  \begin{align}
    \sup_{\|V-W_0\|\leq B}
    \frac{\hcR^{(i)}(V)}{\hcR^{(j)}(V)}
    \leq e^\tau.
    \label{eq:linearization:B}
  \end{align}
  Crucially, \cref{eq:linearization:B} holds with $V:=Z$, since $B \geq R_Z$ by definition.
  Additionally, by \Cref{fact:shallow:magic},
  the following inequality holds \emph{unconditionally}
  for every $j \leq t$:
  \begin{align}
    \|W_j - Z\|^2 + 2\eta \sum_{i<j} \hcRi(W_{i+1})
    \leq
    \|W_0 - Z\|^2 + 2\eta \sum_{i<j} \hcRi(Z).
    \label{eq:magic:B}
  \end{align}
  The remainder of the proof is broken into three parts, for the three separate guarantees:
\begin{align}
    \|W_{\leq t} - W_0\|
    &\leq B
    &\text{(norm)}
    ,
    \label{eq:hmm:A}
    \\
    \hcR(W_{\leq t})
    &\leq e^{2\tau} \hcR^{(0)}(Z) + e^{\tau} (\rho R_Z)^2 \epsopt
    &\text{(empirical risk)},
    \label{eq:hmm:B}
    \\
    \cR(W_{\leq t})
&\leq e^{4\tau} \cR^{(0)}(Z) + e^{3\tau} (\rho R_Z)^2 \epsopt + e^{4\tau}\rho (B + R_Z) \tau_n
    &\text{(risk)}.\label{eq:hmm:C}
  \end{align}

  \paragraph{Norm guarantee (cf. \cref{eq:hmm:A}).}
  There are two cases to consider: $B = \radopt$, or $B < \radopt$.
  If $B=\radopt$, the claim follows by the definition of $W_{\leq t}$.

  Now suppose $B<\radopt$, meaning $B = 3R_Z + 2e \sqrt{\eta t \hcR^{(0)}(Z)}$.
  It will now be argued via contradiction that $\max_{i \leq t}\|W_i - W_0\| \leq B$.
  Assume contradictorily the claim does not hold,
  and let $s\leq t$ be the earliest violation.  But that means the claim holds for all $i<s$,
  which also means,
  combining \cref{eq:linearization:B} (which must hold for all $i<s$)
  and \cref{eq:magic:B} and using $\tau \leq 2$ and $\ell\geq0$,
  \begin{align*}
    B^2 < \|W_s - W_0\|^2
    &\leq 2\|W_s - Z\|^2 + 2\|Z - W_0\|^2
    \\
    &\leq
    2 \|W_s - Z\|^2 + 4\eta \sum_{i<s} \hcRi(W_{i+1})  + 2\|Z - W_0\|^2
    \\
    &\leq
    4\|W_0 - Z\|^2 + 4\eta \sum_{i<s} \hcRi(Z)
    \\
    &\leq
    4\|W_0 - Z\|^2 + 4\eta t e^2 \hcR^{(0)}(Z)
    \\
    &\leq
    \del{2\|W_0 - Z\| + 2e \sqrt{\eta t \hcR^{(0)}(Z)}}^2
    \leq
    B^2,
  \end{align*}
  a contradiction.

  \paragraph{Empirical risk guarantee (cf. \cref{eq:hmm:B}).}
  Now let $T$ denote the earliest time when $\|W_i - W_0\| > 2R_Z$,
  or $T=\infty$ if this situation never occurs.
  Note that for any $i<T$,
  \[
    \|W_i - W_0 \|
    \leq 2R_Z \leq B,
  \]
  and even for $W_T$,
  \[
    \|W_T - W_0\|
    \leq
    \|W_{T-1} - W_0\| + \eta \|\nhcR(W_{T-1})\|
    \leq
    2R_Z + \eta \rho \leq B;
  \]
  as such,
  \cref{eq:linearization:B} holds for all $W_i$ with $i\leq T$, including the edge
  case $W_T$.  The remainder of the proof divides into two cases:
  either $T > t$ (which includes the situation $T = \infty$), or $T \leq t$.

  If $T\leq t$,
  by the triangle inequality,
  \[
    2\|Z - W_0\|
    <
    \|W_T - W_0\|
    \leq 
    \|Z - W_T\| + \|Z - W_0\|,
  \]
  which rearranges to give $\|Z-W_0\| < \|Z-W_T\|$, and thus, by \cref{eq:magic:B},
  \begin{align*}
    \|Z-W_0\|^2 + 2\eta\sum_{i<T} e^{-\tau}\hcR(W_{i+1})
    &<
    \|W_T - Z\|^2 + 2\eta\sum_{i<T} \hcRi(W_{i+1})
    \\
    &\leq
    \|Z - W_0\|^2 + 2\eta\sum_{i<T} \hcRi(Z)
    \\
    &\leq
    \|Z - W_0\|^2 + 2\eta\sum_{i<T} e^\tau \hcR^{(0)}(Z),
  \end{align*}
  which after canceling from both sides and using the definition of $W_{\leq t}$,
  \[
    \hcR(W_{\leq t})
    \leq
    \min_{i < T} \hcR(W_i)
    \leq 
    \frac 1 T \sum_{i<T} \hcR(W_i)
    \leq
    e^{2\tau} \hcR^{(0)}(Z),
  \]
  establishing \cref{eq:hmm:B} when $T \leq t$.

  If $T > t$, the proof is simpler: since $\max_{i\leq t} \|W_i-W_0\|\leq 2R_Z \leq B$,
  then \cref{eq:linearization:B} holds for all $W_i$ with $i\leq t$,
  and thus by \cref{eq:magic:B} and the definition
  of $W_{\leq t}$,
  \begin{align*}
    2\eta e^{-\tau} \sum_{i<t} \hcR(W_{i+1})
    \leq
    2\eta \sum_{i<t} \hcRi(W_{i+1})
    &\leq
    \|W_t - Z\|^2
    +
    2 \eta \sum_{i<t} \hcRi(W_{i+1})
    \\
    &\leq
    \|W_0 - Z\|^2
    +
    2 \eta \sum_{i<t} \hcR^{(i)}(Z)
    \\
    &\leq
    \|W_0 - Z\|^2
    +
    2 t\eta e^\tau \hcR^{(0)}(Z),
  \end{align*}
  which after rearranging and using the definition of $W_{\leq t}$ gives
  \[
    \hcR(W_{\leq t})
    \leq
    \frac 1 t \sum_{i<t} \hcR(W_{i+1})
    \leq
    e^{2\tau} \hcR^{(0)}(Z)
    +
    \frac {e^\tau \|W_0 - Z\|^2}{2 t \eta}
    \leq
    e^{2\tau} \hcR^{(0)}(Z)
    + e^\tau (\rho R_Z)^2 \epsopt,
  \]
  completing the proof of \cref{eq:hmm:B}.

  \paragraph{Risk guarantee (cf. \cref{eq:hmm:C}).}
  By \Cref{fact:shallow:gen:frob:1} applied once with radius $B$ and once with radius $R_Z$,
  with probability at least $1-12\delta$
  \begin{align*}
    \cR^{(0)}(W_{\leq t})
    \leq \hcR^{(0)}(W_{\leq t})
    + \rho B\tau_n,
    \qquad
    \hcR^{(0)}(Z)
    \leq \cR^{(0)}(Z)
    + \rho R_Z\tau_n.
  \end{align*}
  Moreover, by the last part of \Cref{fact:shallow:linearization:frob} applied with
  radius $B$
  with probability at least $1-4\delta$,
  \[
      \cR(W_{\leq t})
      \leq
      e^{\tau} \cR^{(0)}(W_{\leq t}).
  \]
  Combining all these inequalities with the empirical risk guarantee,
  \begin{align*}
    \cR(W_{\leq t})
    &\leq e^{\tau}\cR^{(0)} (W_{\leq t})
    \\
    &\leq e^{\tau}\hcR^{(0)}(W_{\leq t}) + e^\tau \rho B \tau_n
    \\
    &\leq e^{2\tau}\hcR(W_{\leq t}) + e^\tau \rho B \tau_n
    \\
    &\leq e^{4\tau}\hcR^{(0)}(Z) + e^{3\tau}(\rho R_Z)^2\epsopt + e^\tau\rho B \tau_n
    \\
    &\leq e^{4\tau}\cR^{(0)}(Z) + e^{3\tau}(\rho R_Z)^2\epsopt + \rho \del{ e^\tau B + e^{4\tau} R_Z }\tau_n,
  \end{align*}
  thus establishing \cref{eq:hmm:C} and completing the proof.
\end{proof}

\subsection{Approximation proofs}

First, the \namecref{fact:shallow:barUi:sample} and proof that we can sample from $\barUi$;
as the gap is over the risk, the proof uses the technique in \Cref{fact:covc:again}
to control all points on the sphere.
This proof also makes crucial use of the $\arccos$ bound in \Cref{fact:acos}.

\begin{lemma}\label{fact:shallow:barUi:sample}
  Let $\barUi$ be given with $R:=\sup_{v\in\R^d} \|\barUi(v)\|$,
  and suppose $m\geq \ln(emd)$.
  With probability at least
  $1 - 6\delta$,
\begin{align*}
    \cR^{(0)}(\barU) &\leq
    e^{\tau} \cR(\barUi),
    \qquad
    \textup{where }
\tau
\leq
6\rho d\ln(emd^2/\delta)
    +
\frac{20R\sqrt{d\ln(em^2d^3/\delta)}}{m^{1/4}}.
  \end{align*}
\end{lemma}

\begin{proof}[Proof of \Cref{fact:shallow:barUi:sample}]
  Throughout this proof, the subscript will be dropped and simply $W := W_0$,
  with rows $(w_j^\T)_{j=1}^m$.

  The bound on $\cR^{(0)}(\barU) - \cR(\barUi)$ follows by showing that
  with probability at least $1-6\delta$,
  \[
    \sup_{\|x\|\leq 1} \envert{ f(x;\barUi) - \ff{0}(x;\barU) }
    \leq \tau,
\]
  and then as usual applying \Cref{fact:logistic:error} and taking an expectation to obtain
  a bound between $\cR^{(0)}(\barU)$ and $\cR(\barUi)$.
  Meanwhile, this intermediate bound is first established for any fixed $x\in\R^d$, and then
  general $\|x\|\leq 1$ are handled via \Cref{fact:covc:again}.

  Fix an example $x\in\R^d$ and failure probability $\delta_0$ to be determined later
  when \Cref{fact:covc:again} is invoked.
  To first calculate the expected difference, note by definition of $\barU$ that
  \begin{align*}
    \bbE \ip{\nf(x;W)}{\barU - W}
    &= \bbE \frac \rho {\sqrt m} \sum_{j=1}^m a_j \ip{\baru_j - w_{j}}{x\1[w_{j}^\T x \geq 0]}
    \\
    &= \frac 1 m \sum_{j=1}^m \bbE \ip{\barUi(w_{j})}{x\1[w_{j}^\T x \geq 0]}
    \\
    &= f(x;\barUi),
  \end{align*}
  whereas
  \[
    \bbE \ip{\nf(x;W)}{W} = \bbE_a \sum_{j=1}^m a_j \bbE_{w_{j}} \srelu(w_{j}^\T x)
    = 0,
  \]
  thus
  \[
    \bbE \ff{0}(x;\barU) = \bbE \del{\ff{0}(x;\barU-W) + \ff{0}(x;W)} = f(x;\barUi).
  \]
  Controlling the deviations (still for this fixed $x$) will also consider the terms separately.
  The term $\ff{0}(x;\barU-W)$ will use McDiarmid's inequality;
  to verify the bounded differences property, consider pairs $(a,W)$ and $(a',W')$ which
  differ in only one element $(a_j',w_j')$, which also defines pairs $\barU$ and $\barU'$
  differing
  in just one $j$, meaning the vectors $\baru_j$ and $\baru_j'$;
  by Cauchy-Schwarz and the definition of $R$,
  \begin{align*}
    &\envert{ \ip{\nf(W)}{U-W} - \ip{\nf(W')}{U'-W'} }
    \\
    &=
    \envert{
    \frac {\rho}{\sqrt m} 
      a_j\ip{\baru_j - w_{j}}{ x \1[ w_{j}^\T  x_j \geq 0]}
    - \frac {\rho}{\sqrt m} a_j'\ip{\baru_j' -w_j'}{ x \1[ (w_j')^\T  x_j \geq 0]} }
    \\
    &=
    \frac 1 m 
    \envert{
      a_j^2\ip{\barUi(w_j)}{ x \1[ w_j^\T  x_j \geq 0]}
    - (a_j')^2\ip{\barUi(w_j')}{ x \1[ (w_j')^\T  x_j \geq 0]} }
    \\
    &\leq
    \frac {2 R\|x\|}{m}.
  \end{align*}
  Thus, by McDiarmid's inequality, with probability at least $1-2\delta_0$,
  \[
    \envert{ \ff{0}(x;\barU) - f(x;\barUi) }
    =
    \envert{ \ff{0}(x;\barU) - \bbE_{a,W}\ff{0}(x;\barU) }
    \leq
    \sqrt{\frac{2R^2\|x\|^2 \ln(1/\delta_0)}{m}}.
  \]
  Meanwhile, the term $\ff{0}(x;W)$ is explicitly controlled in
in the first part
  of \Cref{fact:gaussians:2}:
  with probability at least $1-3\delta_0$,
  \[
    |\ff{0}(x;W)| \leq 2 \rho\|x\| \ln(1/\delta_0).
  \]
  Together, with probability at least $1-5\delta_0$,
  \[
    \envert{\ff{0}(x;\barU) - f(x;\barUi)}
    \leq
    2 \rho \ln(1/\delta_0)
    +
    R  \sqrt{\frac{2\ln(1/\delta_0)}{m}}
    =: r_2.
  \]

  Controlling the behavior for all $\|x\|\leq 1$ simultaneously now relies upon
  \Cref{fact:covc:again}, but invoked to control a single matrix,
  namely choosing $\cS_0:= \{\barU\}$,
  and radius $R_V := R/\rho \geq \|\barU - W\|$.
  For the sake of applying \Cref{fact:covc:again},
  define for any $V\in\R^{m\times d}$ the mapping
  \[
    h_V(x) := f(x;W) - f(x;\barUi),
  \]
  which has no dependence on $V$,
  and note a corresponding function $h\in\cH$ as defined in \Cref{fact:covc:again}
  has the form
  \[
    h(x) = f(x;W) - f(x;\barUi) + \ip{\nf(x;W)}{V-W}
    = \ip{\nf(x;W)}{V} - f(x;\barUi);
  \]
  since $\cS_0 = \{\barU\}$, we only need to check the conditions of \Cref{fact:covc:again}
  for $V=\barU$.
  As above, for any fixed $\|x\|\leq 1$, with probability at least $1-5\delta_0$,
  $|h(x)|\leq r_2$.
  To invoke \Cref{fact:covc:again}, the restricted continuity property must be established.
  Specifically, let $\|x-z\| \leq \eps$ be given, with $\eps>0$ determined later.
  Writing
  \[
    \envert{ h_V(x) - h_V(z) }
    \leq
    \envert{f(x;W) - f(z;W)}
    +
    \envert{f(x;\barUi)- f(z;\barUi)},
  \]
  it suffices to check the restricted continuity property in both terms separately.
  For the first term,
  by \Cref{fact:gaussians:1},
  with probability at least $1-\delta_0$,
  \[
    \|W\|_2 \leq \sqrt{m} + \sqrt{d} + \sqrt{2\ln(1/\delta_0)},
  \]
  whereby the $1$-Lipschitz property of the ReLU over vectors gives
  \[
    \envert{f(x;W) - f(z;W)}
    \leq
    \rho \|\srelu(Wx) - \srelu(Wz)\|
    \leq
    \rho \|W (x-z)\|
    \leq
    \rho \del{ \sqrt{m} + \sqrt{d} + \sqrt{2\ln(1/\delta_0)} } \eps.
  \]
For the other term, first note by a standard Gaussian calculation that
  \begin{align*}
    \envert{f(z;\barUi) - f(x;\barUi)}
    &= \envert{ \int \ip{\barUi(v)}{z\1[v^\T z \geq 0] - x\1[v ^\T x\geq 0]}\dif\cN(v) }
    \\
    &\leq
    R \int \enVert{z\1[v^\T z \geq 0] - x\1[v ^\T x\geq 0]}\dif\cN(v)
    \\
    &\leq
    R\|z-x\|\Pr_{v\sim \cN}\sbr{ \1[v^\T z\geq 0] = \1[v^\T x\geq 0] }
    \\
    &\qquad
    + R (\|x\| + \|z\|) \Pr_{v\sim \cN}\sbr{ \1[v^\T z\geq 0] \neq \1[v^\T x\geq 0] }
    \\
    &\leq
    R\|z-x\|
    + R (\|x\| + \|z\|) \frac{2\arccos(\ip{x/\|x\|}{z/\|z\|})}{2\pi}.
  \end{align*}
  If $\|x\| \leq 2\eps$, then $z\leq 3\eps$, and the last term can be upper bounded as $5R\eps$.
  On the other hand, if $\|x\|>2\eps$,
  whereby $\|x\| + \|z\| \leq 2\|x\| + \eps \leq 3\|x\|$,
  then \Cref{fact:acos} implies
  \[
    R (\|x\| + \|z\|) \frac{2\arccos(\ip{x/\|x\|}{z/\|z\|})}{2\pi}
    \leq R (\|x\| + \|z\|) \frac{\eps\sqrt{8}}{\|x\|\pi} \leq 3R\eps,
  \]
Thus, by \Cref{fact:covc:again} with radius $R_V := R/\rho$
  and filter set $\cS_0 := \{\barU\}$ as above,
  and additionally choosing $\eps := 1/(dm)$,
  with overall probability at least $1 - (1+5(\sqrt{d}/\eps)^d)\delta_0$,
  \begin{align*}
    \sup_{\|x\|\leq 1}\envert{\ff{0}(x;\barU) - f(x;\barUi)}
    &\leq 
    2\rho\ln(e/\delta_0)
    +
    R  \sqrt{\frac{2\ln(e/\delta_0)}{m}}
    \\
    &\quad +
    \eps \rho \del{\sqrt{m} + \sqrt{d} + \sqrt{2\ln(e/\delta_0)}}
    + (1+5) R\eps
    \\
    &\quad +
11 R_V \rho \del{\frac{\ln(edm/\delta_0)}{m}}^{1/4}
    \\
    &\leq
    6\rho\ln(e/\delta_0)
    +
    20 R_V \rho \del{\frac{\ln(edm/\delta_0)}{m}}^{1/4},
  \end{align*}
  and the final bound comes via the choice
  $\delta_0 := \delta/(md^2)^d$.
\end{proof}

The next result establishes that for any $p_y$, there exists a conditional probability
model defined by $\barUi$ which is arbitrarily close, which is one of the keys to the
consistency proof (cf. \Cref{fact:consistency}).
As discussed briefly in \Cref{rem:main}, this construction requires a bias term,
which is simulated by replacing the input $x\in\R^d$ with $(x,1)/\sqrt{2}\in\R^{d+1}$,
and otherwise proceeding without modification.

\begin{lemma}\label{fact:shallow:barUi:lusin}
  Suppose $\mu_x$ and $p_y$ are Borel measurable, and $\mu_x$ is supported on $\|x\|\leq 1$.
  Given any $x\in\R^d$,
let $\tilde x := (x, 1)/\sqrt{2} \in\R^{d+1}$
  denote the vector
  obtained by appending the constant $1$.
  Then for any $\eps>0$,
  there exist infinite-width weights $\barUi:\R^{d+1}\to\R^{d+1}$ satisfying
  $R:=\sup_{\tv\in\R^{d+1}} \barUi(\tv) < \infty$ and
  \[
    \cR(\barUi) \leq \barcR + \eps.
  \]
\end{lemma}

\begin{proof}Throughout this proof, define $\tau := \min\{\eps/4,1/2\}$.

  As is standard in the theory of classification calibration
  \citep{zhang_convex_consistency,bartlett_jordan_mcauliffe},
  for the logistic loss,
  the optimal population risk is achieved by a measurable function $\bar f:\R\to\bar \R$
  which satisfies
  \[
    \bar f(x) := \argmin_{r \in \R\cup{\pm\infty}} p_y(x) \ell(r) + (1-p_y(x))\ell(-r)
    = \phi^{-1}(p_y(x)) = \ln \frac {p_y(x)}{1-p_y(x)}\qquad \mu_x\textup{-a.e. }x,
  \]
  which may take on the values $\pm\infty$.
  To avoid these $\pm\infty$,
  define a clamping of $p_y$ as
  \[
    p_1(x) := \max\{ \tau, \min\{ 1-\tau, p_y(x) \}\},
  \]
  and clamped logits $f_1(x) := \phi^{-1}(p_1(x))$ (which now is bounded).  As is again usual
  in the literature on classification calibration
  \citep{zhang_convex_consistency,bartlett_jordan_mcauliffe},
  \begin{align*}
    \cR(f_1) - \barcR
    &= \int \del{ p_y(x) \ln \frac {p_y(x)}{p_1(x)} + (1-p_y(x)) \ln \frac {1-p_y(x)}{1-p_1(x)}}\dif\mu_x(x)
    \\
    &=
    \int_{p_y(x)\in[0,\tau)}
    \del{ p_y(x) \ln \frac {p_y(x)}{\tau} + (1-p_y(x)) \ln \frac {1-p_y(x)}{1-\tau}}\dif\mu_x(x)
    \\
    &\quad+
    \int_{p_y(x)\in(1-\tau,1]}
    \del{ p_y(x) \ln \frac {p_y(x)}{1-\tau} + (1-p_y(x)) \ln \frac {1-p_y(x)}{\tau}}\dif\mu_x(x)
    \\
    &\leq \frac {\tau}{1-\tau} \leq 2\tau.
\end{align*}

Since $p_1$ is Borel measurable (due to Borel measurability of $p_y$),
  then $f_1$ is Borel measurable (since $\phi^{-1}$ is continuous along $[\tau,1-\tau]$),
  and therefore we may apply Lusin's Theorem \citep[Theorem 7.10]{folland}:
  there exists a continuous function $g$
  and a set $S$ satisfying
  \[
    |g| \leq |f_1|\leq \sup_x |f_1(x)| < \infty,
    \qquad
    g_{|S} = (f_1)_{|S},
    \qquad
    \mu_x(S^c) \leq \frac {\tau}{\ell(0) + \sup_x |f_1(x)|},
  \]
  whereby since $\ell$ is $1$-Lipschitz,
  \begin{align*}
    \cR(g) - \cR(f_1)
    &\leq \int \1[x \in S^c] \ell(-yg(x))\dif\mu(x,y)
    \\
    &\leq
    \int \1[x\in S^c] \ell(|g(x)|)\dif\mu_x(x)
    \\
    &\leq \mu_x(S^c)(\ell(0) + \sup_x |g(x)|)
    \\\
    &\leq \tau.
  \end{align*}
  Since $g$ is continuous, it is uniformly continuous over $\|x\|\leq 1$,
  and thus there exists a $\delta>0$ so that the modulus of continuity $\omega_g(\delta)$
  at scale $\delta$ is at most $\tau$, meaning
  \[
    \sup_{\|x-x'\|\leq \delta} | g(x) - g(x') | \leq \omega_g(\delta) \leq \tau.
  \]
  By results in neural network universal approximation
  \citep[Theorem 4.3]{ntk_apx}, there exists
infinite-width weights $\barUi:\R^{d+1}\to\R^{d+1}$ satisfying
  $R:=\sup_{\tx} \|\barUi(\tx)\| < \infty$
  and
  \[
    \sup_{\|x\|\leq 1} \envert{f(\tx;\barUi) - g(x)} \leq \omega_g(\delta) \leq \tau,
  \]
  which again by the $1$-Lipschitz property of $\ell$ means $\cR(\barUi) - \cR(g)\leq \tau$.
  Combining all these pieces,
  \begin{align*}
    \cR(\barUi)
    -\barcR
    =
    \sbr{\cR(\barUi) - \cR(g)}
    +
    \sbr{\cR(g) - \cR(f_1)}
    +
    \sbr{\cR(f_1) - \barcR}
    \leq
    \tau + \tau + 2\tau
    \leq
    \eps,
  \end{align*}
  as desired.
\end{proof}

\subsection{Proofs of main results: \Cref{fact:main} and \Cref{fact:consistency}}

The proof of \Cref{fact:main} and a precise restatement are as follows.  This restatement
has fully explicit constants,
and is invoked in the proof of \Cref{fact:consistency} to ease sanity-checking.

\begin{theorem}[Refined restatement of \Cref{fact:main}]\label{fact:main:locking:2}
  Let temperature $\rho > 0$ and reference model
  $\barUi$ be given with $R := \max\{4, \rho, \sup_v\|\barUi(v)\|\}<\infty$,
  and define a corresponding conditional model $\phi_\infty(x) := \phi(f(x;\barUi))$.
  Let optimization accuracy $\epsopt$ and radius $\radopt\geq R/\rho$
  be given,
  define effective radius
  $B := \min\cbr[1]{\radopt,\ {}\frac{3R}{\rho} + \frac{4e}{\rho} \sqrt{t}\sqrt{e^{\tau_0}\cR(\barUi) + R\tau_n} }$,
  where generalization error $\tau_n$ and additionally linearization error $\tau_1$ 
  and sampling error $\tau_0$ are defined as
  \begin{align*}
    \tau_n &:= \frac {80\del{d \ln(em^2 d^3/\delta)}^{3/2}}{\sqrt n},
    \\
    \tau_1 &:= \frac {100 \rho B^{4/3}\sqrt{d\ln(enm^2d^3/\delta)}}{m^{1/6}},
    \\
    \tau_0
    &:=
6\rho d\ln(emd^2/\delta)
    +
    \frac{20R\sqrt{d\ln(em^2d^3/\delta)}}{m^{1/4}},
  \end{align*}
  where it is assumed $\tau_1\leq 2$ and $m\geq \ln(emd)$.
  Choose step size $\eta := 4/\rho^2$, and run gradient descent
  for $t:=1/(8\epsopt)$ iterations, selecting iterate
  $W_{\leq t} :=\argmin\{\hcR(W_i) : i \leq t, \|W_i-W_0\|\leq \radopt\}$
  with simultaneously small norm and empirical risk.
  Then, with probability at least $1-25\delta$,
  \begin{align*}
& \cR(W_{\leq t}) - \barcR
    &\text{(logistic error)}\phantom{.}
    \\
  \leq\qquad& \klb(p_y, \phi_{\infty}) + \del[1]{e^{\tau_1 + \tau_0} - 1}\cR(\barUi)
            &\text{(reference model error)}\phantom{.}
            \\
  + \quad & e^{\tau_1} R^2 \epsopt
          &\text{(optimization error)}\phantom{.}
          \\
  + \quad &
e^{\tau_1} (\rho B + R)\tau_n
          & \hspace{2em}\text{(generalization error)},\\
          \intertext{where the classification and calibration errors satisfy}
& {}\cR(W_{\leq t}) - \barcR
    &\text{(logistic error)}\phantom{.}
    \\
            \geq \qquad& 2\int \del{\phi(f(x;W_{\leq t})) - p_y}^2\dif\mu_x(x)
               &\text{(calibration error)}\phantom{.}
               \\
   \geq \qquad &\frac 1 2 \del{ \cRz(W_{\leq t}) - \barcRz}^2
                &\text{(classification error)}.
  \end{align*}
  Lastly, for any $\epsilon > 0$, there exists $\barUi^{(\eps)}$
  with $\sup_v \|\barUi^{(\eps)}(v)\|<\infty$
  and whose conditional model $\phi_{\infty}^{(\eps)}(x) := \phi(f((x,1)/\sqrt{2};\barUi^{(\eps)}))$
  satisfies $\klb(p_y, \phi_{\infty}^{(\eps)}) \leq \epsilon$.
\end{theorem}
\begin{proof}[Proof of \Cref{fact:main} and simultaneously \Cref{fact:main:locking:2}]
  This proof focuses on the first inequality, upper bounding $\cR(W_{\leq t})-\barcR$;
  for the other two statements,
  the chain of inequalities with other error metrics are from \Cref{fact:logistic:error},
  and the approximation of arbitrary Borel measurable $p_y$ is from
  \Cref{fact:shallow:barUi:lusin}.
  (The only difference between \Cref{fact:main:locking:2} here and
  \Cref{fact:main} in the body
  is that the ``$\tcO$'' hides constants and $\ln(m)$ and $\ln(d)$
  (but not $\ln(n)$).

  Returning to the first inequality,
  let $\barU$ be the canonical sample of $\barUi$ as in \cref{eq:barUi:sample},
  where $\|\barU - W_0\|\leq R/\rho$ by construction.
  By \Cref{fact:shallow:barUi:sample},
  with probability at least $1 - 6\delta$,
  then $\cR^{(0)}(\barU) \leq e^{\tau_0} \cR(\barUi)$,
  where $\tau_0$ is as in the statement (cf. \Cref{fact:main:locking:2}).

  Next instantiate \Cref{fact:shallow:magic:linearized}
  with reference matrix $Z = \barU$ and $R_Z := R/\rho$,
  whereby the definition of $R$ gives $R_Z \geq \{1,\eta\rho,\|\barU-W_0\|\}$ as needed;
  as such, ignoring an additional failure probability at most $19\delta$,
  setting $\tau:=\tau_1/4$ in the invocation,
  and lastly subtracting $\barcR$ from both sides,
  \begin{align*}
    \cR(W_{\leq t}) - \barcR
    &\leq
    e^{\tau_1} \cR^{(0)}(\barU)
    + e^{\tau_1} (\rho R_Z)^2 \epsopt
    +
e^{\tau_1} (\rho B + \rho R_Z)\tau_n
    -\barcR
    \\
    &\leq
    \del{e^{\tau_1 + \tau_0} - 1} \cR(\barUi)
    + \klb(p_y, \phi_{\infty})
    + e^{\tau_1} R^2 \epsopt
+
e^{\tau_1} (\rho B + R)\tau_n.
  \end{align*}
  This invocation of \Cref{fact:shallow:magic:linearized} also guarantees
  $\hcR^{(0)}(\barU) \leq \cR^{(0)}(\barU) + R \tau_n$
  which together with the earlier inequality 
  $\cR^{(0)}(\barU) \leq e^{\tau_0} \cR(\barUi)$ provides the form of $B$ used in the statement
  (this $B$ upper bounds the one defined in \Cref{fact:shallow:magic:linearized}, which is fine
  since it only relaxes the guarantees provided there).
\end{proof}

Making use of \Cref{fact:main:locking:2}, the proof of the consistency statement,
\Cref{fact:consistency}, is as follows.
Note that we are always working with bias-augmented inputs within this statement and
its proof; e.g., $\widehat W_n\in\R^{m^{(n)}\times (d+1)}$.

\begin{proof}[Proof of \Cref{fact:consistency}]
  Let $\eps>0$ be arbitrary, and define the event
  \[
    E_n := \sbr{ \cR(\widehat W_n) \geq \barcR + \eps }.
  \]
  Following a standard scheme for consistency proofs
  \citep[Corollary 12.3]{schapire_freund_book_final},
  it suffices, thanks to the Borel-Cantelli lemma, to prove
  \begin{equation}
    \sum_{n\geq 1} \Pr[E_n] < \infty;
    \label{eq:consistency:b-c}
  \end{equation}
  that is to say, by the Borel-Cantelli lemma, \cref{eq:consistency:b-c} implies
  $\limsup_{n\to \infty} \cR(\widehat W_n) - \barcR \leq \eps$ almost surely,
  and since $\cR(\widehat W_n) \geq \barcR$ and since $\eps>0$ was arbitrary,
  it follows that $\cR(\widehat W_n) \to \barcR$ almost surely.
  Moreover, by \Cref{fact:logistic:error},
  for each $n$ there are the inequalities
  \[
    \frac 1 2 \del{ \cRz(\widehat W_n) - \barcRz}^2
    \leq 2\int (\widehat\phi_n(x) - p_y(x))^2\dif\mu_x(x)
\leq
    \cR(\widehat W_n) - \barcR,
  \]
  thus $\cR(\widehat W_n)\to\barcR$ also implies $\widehat\phi_n \to p_y$ in $L_2(\mu_x)$
  almost surely,
  and $\cRz(\widehat W_n) \to \barcRz$ almost surely.

  To establish \cref{eq:consistency:b-c},
  first use the last part of \Cref{fact:main:locking:2} to fix a $\barUi$ 
  with $\klb(p_y, \widehat\phi_n) \leq \eps/2$,
  and define $R:=\sup_v \|\barUi(v)\|<\infty$.
  To bound $\Pr[E_n]$, instantiate \Cref{fact:main:locking:2} for every $n$ with 
  reference model $\barUi$ and corresponding $R<\infty$,
  and failure probability $\delta^{(n)} := 1/n^2$,
  and optimization radius $\radopt = \infty$, meaning a corresponding effective radius
  given by \Cref{fact:main:locking:2} as
  \[
    B^{(n)} = \frac {1}{\rho^{(n)}}\del{3R + 4e\sqrt{t^{(n)}}\sqrt{e^{\tau_0^{(n)}}\cR(\barUi) + R \tau_n}}
    .
\]
  Inspecting all the terms in \Cref{fact:main:locking:2}, it will now be argued that
  while the term $\klb(p_y,\widehat\phi_n)$ stays level and is at most $\eps/2$ independent of $n$,
  all other terms go to $0$.  Returning to
  $B^{(n)}$, since $\tau_n = \tcO(1/\sqrt{n})$ and $\tau_0^{(n)} \to 0$ (which will be shown later), then
  $B^{(n)}=  \tcO(\sqrt{t^{(n)}}/\rho^{(n)})$,
  whereby
  \begin{align*}
    \tau_1^{(n)}
    &=
    \tcO\del{\frac {\rho^{(n)} (B^{(n)})^{4/3}}{(m^{(n)})^{1/6}}}
    =
    \tcO\del{\frac {(t^{(n)})^{2/3}}{(m^{(n)})^{1/6}(\rho^{(n)})^{1/3}}}
    \\
    &=
    \tcO\del{\frac {(t^{(n)})^{2/3}}{(m^{(n)})^{1/8}}}
    =
    \tcO\del{\frac {n^{\frac{2}{3}(1-\xi)}}{n^{\frac{5}{3}(1-\xi)}}}
    =
    \tcO\del{n^{\xi-1}}\to 0.
  \end{align*}
  Next,
  \[
    \tau_0^{(n)}
    =
    \tcO\del{\rho^{(n)} + \frac {1}{(m^{(n)})^{1/4}}}
    =
    \tcO\del{ n^{\frac{5}{3}(\xi-1)} + n^{\frac{10}{3}(\xi-1)}} \to 0,
  \]
  which together with the asymptotics of $\tau_1^{(n)}$
  gives 
  $\exp(\tau_0^{(n)} + \tau_{1}^{(n)}) - 1 \to 0$ 
  and $\exp(\tau_1^{(n)})R^2\epsopt^{(n)}\to 0$.
  The final term to consider is
  \[
    \exp(\tau_1^{(n)})\rho^{(n)}B^{(n)}\tau_n
    = \tcO\del{ \sqrt{\frac{t^{(n)}}{n}}}
    = \tcO(n^{- \xi/2})
    \to 0.
  \]
  As such, all terms go to zero with $n$ (excepting $\klb(p_y,\widehat\phi_n)\leq \eps/2$,
  which is fine),
  and there exists $N_0$ so that for all $n> N_0$,
  all conditions of the bound are met, and with the exclusion of a failure probability of
  $\delta^{(n)}$, the bound implies $\cR(\widehat W_n) < \barcR + \eps$.
  Thus $n\geq N_0$ implies $\Pr[E_n] \leq \delta^{(n)} = 1/n^2$,
  and
  \[
    \sum_{n\geq 1} \Pr[E_n]
    \leq
    \sum_{n \leq N_0} 1
    + \sum_{n > N_0} \frac 1 {n^2}
    \leq
    N_0 + \frac {\pi^2}{6}
    < \infty,
  \]
  which establishes \cref{eq:consistency:b-c} and completes the proof.
\end{proof}

\section{Proof of \Cref{fact:lb:local}}

\Cref{fact:lb:local} is a consequence of the following more refined statement,
which also suggests the method of proof, and is consistent with \Cref{fig:interp}.

\begin{lemma}
  \label{fact:lb:pairs}
  Suppose marginal distribution $\mu_x$ is continuous and compactly supported
  on $[0,1]$,
  $p_y$ is continuous, and that either
  $\mu_x(p_y^{-1}((0,1/2))) > 0$ or $\mu_x(p_y^{-1}((1/2,1)) > 0$,
  meaning $p_y$ is outside $\{0,1/2,1\}$ on a set which has positive measure according
  to $\mu_x$.

  Then there exists a constant $c\in (0,1/4)$ (depending only on $\mu_x$ and $p_y$)
  so that with probability at least $1-7\delta$
  over the draw of $((x_i,y_i))_{i=1}^n$ with $n \geq \ln(1/\delta)/c$,
  there exists an interval $I\subseteq [0,1]$,
  and a subset of pairs of indices indices $S \subseteq [m]^2$ satisfying the following properties.
  \begin{enumerate}
    \item
      Either $p_y \in [c,1/2-c]$ everywhere on $I$, or $p_y \in [1/2+c,1-c]$ everywhere on $I$;
      henceforth let $\hat y := \sgn(p_y-1/2)$ designate the correct (Bayes) prediction
      over $I$.

    \item
      If $(i,k)\in S$, then $x_i < x_k = \min\{ x_s : x_s \geq x_i\}$,
      meaning $x_k$ is the first point to the right of $x_i$,
      and moreover the corresponding labels $y_i = y_k = -\hat y$ agree with each
      other but are incorrect.

    \item
      For any local interpolation rule $f\in\cF_n$ (cf. \Cref{fact:lb:local}),
      \[
        \cRz(f) \geq \bcRz + c.
      \]
  \end{enumerate}
\end{lemma}
\begin{proof}[Proof of \Cref{fact:lb:pairs} (and simultaneously \Cref{fact:lb:local})]
  Consider any point $x$ where $p_y(x) \not\in\{0,1/2,1\}$ and $\mu_x > 0$;
  such a point must exist by the assumptions.
  Define $\hat y := \sgn(p_y(x) - 1/2)$ and $c_1 := \min\{p_y(x)/2, |p_y(x)-1/2|/2, (1-p_y(x))/2\}$,
  where $c_1\in (0,1/4)$ by construction.
  Since $p_y$ and $\mu_x$ are continuous, then there must exist some (potentially tiny)
  closed interval $I$ containing $x$ so that $\sgn(p_y(x) - 1/2) = \hat y$, and for any $x' \in I$,
  both $\mu_x(x') > 0$ and $p_{x'} \in (c_1,1/2-c_1) \cup (1/2+c_1,1-c_1)$.

  To simplify the rest of the proof, suppose $\hat y = -1$; the other case is symmetric, but
  as in the preceding paragraph, handling both cases simultaneously adds significant notational
  overhead.

  Let $S$ denote all adjacent pairs of points in $I$ where $(x_i,x_k)\in S$ means
  $x_i < x_k = \min\{x_s : x_s > x_i\}$ and $y_i = y_k = -\hy$.
  With this choice, all that remains to be shown is the third item, the lower bound on the risk.To show this, it suffices to show that a constant fraction of $\mu_x$'s probability
  mass is contained between these pairs, meaning
  \[
    \mu_x\del{\cup_{(i,k)\in S} \mu([x_i,x_k]) } \geq c_2 > 0,
  \]
  where crucially $c_2$ is independent of $n$.  To see that this suffices to establish the
  third property, suppose that $f:\R\to\R$ satisfies the required condition,
  meaning $f(x) \hy < 0$ for $x \in \cup_{(i,k)\in S} \mu([x_i,x_k])$;
  then by a standard calculation against the Bayes risk
  \citep{DGL},
  \begin{align*}
    \cRz(f) - \bcRz
    &= \int |1-2p_y(x)| \1\sbr{\sgn(f) \neq \sgn(p_y(x)-1/2)} \dif\mu_x(x)
    \\
    &\geq
    \int |1-2p_y(x)| \1\sbr{x \in \cup_{(i,k)\in S}[x_i,x_k]} \dif\mu_x(x)
    \\
    &\geq
    2c_1 
    \mu_x\del{\cup_{(i,k)\in S}[x_i,x_k]}
    \\
    &=
    2c_1 c_2,
  \end{align*}
  and the final statement and all properties are satisfied
  if we pick $c\in \big(0, \min\{c_1,c_2,2c_1c_2\}\big]$.

  As such, it remains to provide a lower bound on $c_2$ which is independent of $n$,
  which will follow a series of simplifications as follows.

  The first step is to lower bound the cardinality of $S$.
  The expected number of points in $I$ is $n \mu_x(I)$,
  and if $n \geq 32 \ln (1/\delta)/\mu_x(I)$,
  then by a multiplicative Chernoff bound \citep[Theorem 12.6]{blum_hopcroft_kannan},
  with probability at least $1-3\delta$,
  \[
    \envert{\cbr{ i \in [m] : x_i \in I}} \geq \frac {n \mu_x(I)}{2}.
  \]
  and thus the number of consecutive pairs in $I$ is at least $n\mu_x(I)/2 - 1 \geq n\mu_x(I)/4$.

  Since these pairs may share endpoints, consider the set of at least $n\mu_x(I)/8$ pairs that
  share no points.
  Since the draw of $y$ is independent of $x$, for each of these consecutive pairs,
  the probability that both labels are wrong is at least $(1-c_1)^2$ (and is independent of
  other pairs),
  meaning the expected number of such points is at least $n\mu_x(I)(1-c_1)^2/8$;
  as such, if $n\geq 256 \ln(1/\delta)/(\mu_x(I)(1-c_1)^2)$,
  by another multiplicative Chernoff bound,
  with probability at least $1-3\delta$,
  the number of pairs with agreeing but incorrect labels is at least $n\mu_x(I)(1-c_1)^2/16$.
  Let $S_0$ denote this set of pairs;
  by construction, its cardinality also lower bounds that of $S$.

  It remains to show that the union of the convex hulls of these pairs of points has a significant
  fraction of total probability mass.

  For any sample $(x_1,\ldots,x_n)$, let $(x_{(1)}, \ldots, x_{(n)})$
  be the sample in sorted order,
  meaning $x_{(1)} < x_{(2)} < \cdots < x_{(n)}$ (strict inequalities almost surely
  since $\mu_x$ is continuous).
  Define a distance $\Delta$ and function $F$ of the sample as
  \begin{align*}
    \Delta 
    &:= \frac {\mu_x(I)(1-c_1)^2}{256 n},
    \\
    F(x_1,\ldots,x_n)
    &:= \envert{ \cbr{ i \in [m-1] : \mu([x_{(i)}, x_{(i+1)}]) < \Delta} };
  \end{align*}
  that is to say, $F$ measures the number of consecutive pairs whose convex hulls have probability
  mass strictly less than $\Delta$.  As will be established momentarily,
  $F$ satisfies the bounded differences property
  with a constant $2$, meaning for any two samples $(x_1,\ldots,x_n)$ and $(x_1',\ldots,x_n')$
  that differ only in a single example $x_i\neq x_i'$,
  \[
    \envert{F(x_1,\ldots,x_n) - F(x_1',\ldots,x_n')} \leq 2.
  \]
  To argue this, suppose the disagreeing example $x_i$ occupies position $j$ after sorting,
  meaning $x_i = x_{(j)}$, and consider adjusting one sample to the other by renaming
  this point to $x_i'$, removing it from its current location, and moving it to its final
  location.
  \begin{itemize}
    \item First we remove $x_i'$ from the interval $(x_{(j-1)}, x_{(j+1)})$. If neither
      $(x_{(j-1)}, x_i')$ nor $(x_i', x_{(j+1)})$ counts towards $F$, then neither will
      $(x_{(j-1)}, x_{(j+1)})$, so $F$ remains unchanged. If exactly one of
      $(x_{(j-1)}, x_i')$ and $(x_i', x_{(j+1)})$ counts towards $F$, then
      $(x_{(j-1)}, x_{(j+1)})$ does not count towards $F$, so $F$ decreases by 1.
      If both $(x_{(j-1)}, x_i')$ and $(x_i', x_{(j+1)})$ counts towards $F$, then
      $(x_{(j-1)}, x_{(j+1)})$ may or may not count towards $F$, so $F$ decreases by 1 or 2.
      So this operation changes $F$ by any of $\{-2,-1,0\}$.
    \item Then we insert $x_i'$ into a new interval. The range of possible changes to $F$
      is the exact opposite as removing it from an interval, so this leads to a change by
      any of $\{+2,+1,0\}$;
      together the difference in $F$ is within $[-2,+2]$.
  \end{itemize}
  As such, by McDiarmid's inequality, with probability at least $1-\delta$,
  \[
    F(x_1,\ldots,x_n) \leq \bbE F(x_1,\ldots,x_n) + \sqrt{2n\ln(1/\delta)}.
  \]
  Upper bounding $\bbE F(x_1,\ldots,x_n)$ can now be performed in a coarse way as follows.
  Partition the support of $\mu_x$, $[0,1]$, into two systems of intervals,
  $\cI$ and $\cJ$, as follows.
  $\cI$ simply contains the $\left\lceil 1/(2\Delta) \right\rceil$ consecutive intervals
  of mass $2\Delta$ (except for the last, which may have less mass);
  meanwhile, $\cJ$ contains a first initial interval of mass $\Delta$, and then intervals
  of mass $2\Delta$ until a final interval of mass at most $2\Delta$.  Due to this staggered behavior,
  if some pair $(x_{(i)}, x_{(i+1)})$ has $\mu_x((x_{(i)},x_{(i+1)}))<\Delta$,
  then the pair must appear in a single interval in either $\cI$ or $\cJ$ (the staggering
  avoids boundary issues).
  Now consider the creation of the full data sample by sampling the data points one by one,
  and the resulting effect on these bins;
  the goal is to upper bound
  the number of times a point is inserted into an occupied bin,
  as this upper bounds the number of consecutive pairs of points within some bin,
  which in turn upper bounds $F$.
  After inserting the $i$th point (twice),
  let $A_i$ denote the number of occupied bins, and $B_i$ the number of times a point was
  inserted into an occupied bin; necessarily,
  $A_i = 2i - B_i$
  (the factor two coming from simultaneous throws to $\cI$ and $\cJ$).
  The probability of landing in an occupied bin (and thus increasing $B_i$) is at most
  $A_i(2\Delta) = (2i - B_i)(2\Delta)$.
  By linearity of expectation,
  \begin{align*}
    \bbE F \leq \bbE B_n
    &\leq \sum_{i=1}^{n-1} 2 \bbE\1\sbr{ x_{i+1} \textup{ lands in an occupied bin }}
    \\
    &\leq 4\Delta \sum_{i=1}^{n-1} (2i - \bbE B_i)
    \leq 4\Delta (n-1) n
    \leq \frac {n\mu_x(I)(1-c_1)^2}{64}.
  \end{align*}
  Together, supposing that $n\geq 8192 \ln(1/\delta)/(\mu_x(I)^2(1-c_1)^4)$, it follows that with probability at least $1-\delta$,
  \[
    F(x_1,\ldots,x_n) \leq \frac {n\mu_x(I)(1-c_1)^2}{64} + \sqrt{2n\ln(1/\delta)}
    \leq
    \frac {n\mu_x(I)(1-c_1)^2}{32}.
  \]

  To finish the proof, since the preceding quantity is less than half the cardinality
  of $S_0$, we are guaranteed that at least half the pairs in $S_0$
  have $\mu_x((x_i,x_k)) \geq \Delta$; letting $S_1$ denote this half, then
  \begin{align*}
    \mu_x(\cup_{i,k\in S} [x_i,x_k])
    &\geq
    \sum_{(i,k)\in S_1} \mu_x([x_i,x_k])
    \\
    &\geq
    |S_1| \Delta
    \geq
    \frac {n\mu_x(I)(1-c_1)^2}{32} \cdot \Delta
    \geq
    \frac {\mu_x(I)^2 (1-c_1)^4}{8192} =: c_3.
  \end{align*}
  It only remains to determine the final value of the constant $c$.
  By the preceding calculation
  and the comments near the start of the proof establishing that $c\in(0, \min\{c_1,c_2,2c_1c_2\}]$
  suffices, the quantity $c_3$ here is indeed a lower bound on $c_2$, and thus, defining
  $c_4 := \min\{c_1,c_3,2c_1c_3\}$, it suffices to require $c\in (0,c_4]$.
  On the other hand, inspecting all the necessary lower bounds on $n$ throughout the proof,
  the maximum across all of them is that we need $n\geq \ln(1/\delta)/c_3$.
  As such, all properties are satisfied if we take $c:= c_4>0$ as our final constant,
  which depends only on $\mu_x$ and $p_y$ (but not on $n$) as promised.
\end{proof}

\end{document}